\documentclass{article}

\usepackage{amsfonts,amssymb,amsmath,amsthm,subfigure,graphicx}
\usepackage[a4paper]{geometry}
\usepackage{algorithm,algorithmic}
\usepackage{rotating,multirow}
\usepackage[stable]{footmisc}
\usepackage{color}
\usepackage{relsize}

\DeclareRobustCommand\onedot{\futurelet\@let@token\@onedot}
\def\@onedot{\ifx\@let@token.\else.\null\fi\xspace}

\makeatother

\DeclareMathOperator*{\argmin}{arg\,min}

\long\def\ignore#1{}

\def\calE{{\cal E}}

\def\calP{{\cal P}}

\def\calU{{\cal U}}

\def\calX{{\cal X}}
\newcommand{\eqdef}{{\stackrel{\mbox{\tiny \tt ~def~}}{=}}}
\newcommand{\norm}[2][]{\|{#2}\|_{{#1}}}
\newcommand{\scal}[2]{\left\langle #1,#2 \right\rangle}

\newcommand{\setalglineno}[1]{%
  \setcounter{ALC@line}{\numexpr#1-1}}

\newtheorem{theorem}{Theorem}
\newtheorem{lemma}{Lemma}
\newtheorem{proposition}{Proposition}
\newtheorem{corollary}{Corollary}

\begin{document}

\title{Total variation on a tree}

\date{}

\author{Vladimir Kolmogorov\thanks{Vladimir Kolmogorov, Michal
    Rolinek: Institute of Science and Technology Austria (IST
    Austria), 3400 Klosterneuburg, Austria. E-mail:
    \texttt{michalrolinek@gmail.com}, \texttt{vnk@ist.ac.at}. They are
    supported by the European Research Council under the European
    Unions Seventh Framework Programme (FP7/2007-2013)/ERC grant
    agreement no 616160.}, and Thomas Pock\thanks{Thomas Pock:
    Institute for Computer Graphics and Vision, Graz University of
    Technology, 8010 Graz, Austria. Digital Safety \& Security
    Department, AIT Austrian Institute of Technology GmbH, 1220
    Vienna, Austria. E-mail: \texttt{pock@icg.tugraz.at}. Thomas Pock
    is supported by the European Research Council under the Horizon
    2020 program, ERC starting grant agreement 640156.}, and Michal
  Rolinek\footnotemark[1]}

\maketitle

\begin{abstract}
  We consider the problem of minimizing the continuous valued total
  variation subject to different unary terms on trees and propose fast
  direct algorithms based on dynamic programming to solve these
  problems. We treat both the convex and the non-convex case and
  derive worst case complexities that are equal or better than
  existing methods.  We show applications to total variation based 2D
  image processing and computer vision problems based on a Lagrangian
  decomposition approach.  The resulting algorithms are very
  efficient, offer a high degree of parallelism and come along with
  memory requirements which are only in the order of the number of
  image pixels.
\end{abstract}



\section{Introduction}\label{sec:intro}
Consider the following problem:
\begin{equation}
\min_{x\in \mathbb R^n} f(x),\qquad f(x)=\sum_{i\in V} f_i(x_i) + \sum_{(i,j)\in E} f_{ij}(x_j-x_i) 
\label{eq:main}
\end{equation}
where $(V,E)$ is a (directed) tree with $n=|V|$ nodes, unary terms $f_i:\mathbb R\rightarrow\mathbb R$
are continuous functions, and pairwise terms are given by
\begin{equation}
f_{ij}(z)=\min\{w_{ij}\cdot |z|,C_{ij}\}
\label{eq:fij}
\end{equation}
with $w_{ij}\ge 0$. This is known as a ``truncated TV regularizer'';
if $C_{ij}=+\infty$ then it is called a ``TV regularizer''.
To simplify the presentation, we make the following assumptions:
\begin{itemize}
\item {\em Function $f$ is bounded from below and attains a minimum at some point $x\in\mathbb R^n$.}
\item {\em All terms $f_i$, $f_{ij}$ are continuous piecewise-linear or piecewise-quadratic functions with a finite number of breakpoints.}
\end{itemize}

%
%
%
We will consider the following cases.

\paragraph{Non-convex case} 
Here we will assume that unary terms $f_i$ are piecewise-linear functions with $O(1)$ breakpoints
(not necessarily convex).
We will present a dynamic programming algorithm for minimizing function~\eqref{eq:main} that
works by passing {\em messages}, which are piecewise-linear functions $\mathbb R\rightarrow\mathbb R$. 
If $C_{ij}=+\infty$ for all $(i,j)\in E$ then we will prove that
the number of breakpoints in each message is at most $O(n)$, leading to complexity $O(n^2)$.
In the truncated TV case we do not have a polynomial bound.
Our tests, however, indicate that the algorithm is efficient in practice.

\paragraph{Convex case} Next, we will consider the case when all unary and pairwise terms $f_i, f_{ij}$
are convex functions (which means that $C_{ij}=+\infty$ for all $(i,j)\in E$). We will describe three algorithms: 
(i)~$O(n)$ algorithm for quadratic unaries on a chain. 
(ii)~$O(n\log n)$ algorithm for piecewise-linear or piecewise-quadratic unaries on a tree.
(iii)~$O(n\log\log n)$ algorithm for piecewise-linear unaries on a chain.
In the last two cases we assume that the number of breakpoints in each term $f_i$ is $O(1)$.

\subsection{Related work} 
\paragraph{Non-convex case} 
In this case we show how to compute efficiently
{\em distance transforms} (or {\em min-convolutions}) for continuous piecewise-linear functions.
To our knowledge, the previous algorithmic work considered only distance transforms for discretized functions~\cite{Felzenszwalb:TCOMP12}.

\paragraph{Convex case on general graphs} Hochbaum showed~\cite{Hochbaum:ACM01} that problem~\eqref{eq:main}
on general graphs can be solved in polynomial time for several choices of convex unary functions.
The method works by reducing problem~\eqref{eq:main} to a sequence of optimization problems
with {\em binary} variables whose unary terms depend linearly on a parameter $\lambda$.
This reduction
has also appeared later in~\cite{Zalesky:JAM02,Darbon:JMIV:I,Chambolle:EMMCVPR05}.

Specializing Hochbaum's method to trees yields the following complexities:
(i) $O(n^2)$ for problems with quadratic unaries, assuming that the values of $\lambda$ are chosen as in~\cite{ES:76};
(ii) $O(n\log n)$ for piecewise-linear unaries with $O(1)$ breakpoints, assuming that the values of $\lambda$
are computed by a linear-time median algorithm (as discussed in Sec.~\ref{sec:Convex:PiecewiseLinear} for chains).
Instead of using a linear-time median algorithm, it is also possible to sort all breakpoints in $O(n\log n)$ time in a preprocessing step.

\paragraph{Convex case on chains} 
The convex case on a chain (or its continuous-domain version) has been
addressed in~\cite{MammenGeer:97, DaviesKovac:01, Hinterberger:03,
  Steidl:04, Grasmair:07, DeumbgenKovac:09, Condat:13, Johnson2013,
  BarberoSra}.  In particular, it has been shown that the problem with
quadratic unaries $f_i(z)=\frac{1}{2}(z-c_i)^2$ can be solved in
$O(n)$ time by the {\em taut string
  algorithm}~\cite{DaviesKovac:01,DeumbgenKovac:09} and by the method
of Johnson~\cite{Johnson2013}. Condat~\cite{Condat:13} presented an
$O(n^2)$ algorithm, which however empirically outperformed the method
in~\cite{DaviesKovac:01,DeumbgenKovac:09} according to the tests
in~\cite{Condat:13}. In~\cite{BarberoSra}, the authors proposed an
elegant derivation of the method of Condat~\cite{Condat:13} starting
from the tau string algorithm~\cite{DaviesKovac:01}, which in turn
also allows to use weighted total variation. Our $O(n)$ method for
this case can be viewed as a generalization to weighted total
variation and an alternative implementation of Johnson's algorithm
that requires less memory.

For the problem with piecewise-linear unaries $f_i(z)=|z-c_i|$ the
best known complexity was $O(n\log n)$, which is achieved either by Hochbaum's method (as discussed earlier),
or by the method in~\cite{DeumbgenKovac:09}. We improve this to $O(n\log\log n)$.


We generally follow the derivation in~\cite{Johnson2013}, which is quite different from the one
in~\cite{DaviesKovac:01,DeumbgenKovac:09,Condat:13}.
We extend this derivation to non-smooth functions and to general trees.

\subsection{Applications}
In Sec.~\ref{sec:experiments} we show applications to continuous
valued total variation based 2D image processing and computer
vision. For this we adopt a Lagrangian approach to decompose the 2D
problems into a set of 1D problems. In the convex case, we solve the
resulting saddle point problems using accelerated primal-dual
algorithms, outperforming the state-of-the-art by about one order of
magnitude. In the non-convex case we solve a non-convex saddle-point
problem by again applying a primal-dual algorithm, which however has
no theoretical guarantee to converge. The resulting algorithms are
efficient, easy to parallelize and require memory only in the order of
the image pixels.


\section{Preliminaries}\label{sec:background}
We assume that all edges in the tree $(V,E)$ are oriented toward the root $r\in V$.
Thus, every node $i\in V-\{r\}$ has exactly one parent edge $(i,j)\in E$.
When specializing to a chain, we assume that $V=[n]:=\{1,\ldots,n\}$ and $E=\{(i,i+1)\:|\:i\in[n-1]\}$, with $n$ being the root.


\paragraph{Min-convolution}
For functions $h,g:\mathbb R\rightarrow\mathbb R$
we define their {\em min-convolution} $h\otimes g$ via
\begin{equation}
(h\otimes g)(y)=\min_x [h(x)+g(y-x)]\qquad \forall y
\label{eq:MinConvolutionDef}
\end{equation}
With such operation we can associate a mapping $\pi$ that returns $\pi(y)\in\argmin_x [h(x)+g(y-x)]$
for any given $y$; we say that such mapping $\pi$ {\em corresponds} to the min-convolution operation above.\footnote{In this paper
we will apply operation~\eqref{eq:MinConvolutionDef} only in cases in which
it is defined, i.e.\ the minimum exists and is attained at some point $x\in \mathbb R$ (for each $y\in\mathbb R$).
In particular, both functions $h$ and $g$ will be piecewise-linear or piecewise-quadratic.
}

Operation~\eqref{eq:MinConvolutionDef} is known under several other names, e.g.\ the  {\em maximum transform}~\cite{Bellman:62},
the {\em slope transform}~\cite{Maragos:95}, and the {\em distance transform}~\cite{Felzenszwalb:TCOMP12}.
Note that if $g=f_{ij}$  and function $h$ is defined on a grid
then there exist efficient algorithms for computing $h\otimes g$~\cite{Felzenszwalb:TCOMP12}.
For the non-convex case we will need to extend these algorithms to piecewise-linear functions $h:\mathbb R\rightarrow\mathbb R$.

\paragraph{Dynamic Programming}
It is well-known that function~\eqref{eq:main} on a tree can be minimized using a dynamic
programming (DP) procedure (see e.g.~\cite{DP:survey}). If the tree is a chain, then it is equivalent to the {\em Viterbi algorithm}.
Let us review this procedure.

DP works with {\em messages}
$M_{ij}:\mathbb R\rightarrow\mathbb R$ for $(i,j)\in E$ and $\widehat M_i:\mathbb R\rightarrow\mathbb R$ for $i\in V$.
These messages are computed in the forward pass by going through edges $(i,j)\in E$
in the order starting from leaves toward the root and setting
\begin{subequations}\label{eq:Message:update}
\begin{eqnarray}
\widehat M_{i}(x_i)&=&f_i(x_i)+\sum_{(k,i)\in E} M_{ki}(x_i) \label{eq:Message:update:a} \\
M_{ij}(x_j)&=&\min_{x_i} \left[\widehat M_i(x_i) + f_{ij}(x_j-x_i) \label{eq:Message:update:b}\right]
\end{eqnarray}
\end{subequations}
for all $x_i$ and $x_j$. (Due to the chosen order of updates, the right-hand side is always defined).
Note that \eqref{eq:Message:update:b} is a min-convolution operation: $M_{ij}=\widehat M_i\otimes f_{ij}$.
While computing it, we also need to determine a corresponding mapping $\pi_{ij}$ (it will be used in the backward pass).

After computing all messages we first find $x_r$ that minimizes $\widehat M_r(x_r)$,
and then go through edges $(i,j)\in E$ in the backward order and set $x_i=\pi_{ij}(x_j)$.
For completeness, let us show the correctness of this procedure.
\begin{proposition}
The procedure above returns a minimizer of $f(\cdot)$.
\end{proposition}
\begin{proof}
For a node $i\in V$ and an edge $(i,j)\in E$ define ``partial costs'' $f_{\ast i}(\cdot)$ and $f_{\ast ij}(\cdot)$
via
\begin{subequations}
\begin{eqnarray}
f_{\ast i}(x)&=&\sum_{p\in V_i} f_p(x_p) + \sum_{(p,q)\in E_i} f_{pq}(x_q-x_p) \qquad \forall x\in\mathbb R^n \\ 
f_{\ast ij}(x) &=& f_{\ast i}(x) +  f_{ij}(x_j-x_i)                            \qquad \hspace{57pt} \forall x\in\mathbb R^n   \label{eq:HGAIHFIASGFA}
\end{eqnarray}
\end{subequations}
where $(V_i,E_i)$ is the subtree of $(V,E)$ rooted at $i$. In particular, for the root $i=r$ we have  $(V_r,E_r)=(V,E)$
and $f_{\ast r}(x)=f(x)$. Now define functions
\begin{subequations}\label{eq:Message:definition}
\begin{eqnarray}
\widehat M_{i}(z)&=&\min_{x:x_i=z} f_{\ast i}(x) \hspace{4pt} \qquad \forall z\in\mathbb R \label{eq:Message:definition:a} \\
 M_{ij}(z)&=&\min_{x:x_j=z} f_{\ast ij}(x)        \qquad \forall z\in\mathbb R \label{eq:Message:definition:b}
\end{eqnarray}
\end{subequations}
It can be checked that these functions satisfy equations~\eqref{eq:Message:update}: for~\eqref{eq:Message:update:a} 
we can use the fact that $f_{\ast i}(x)=f_i(x)+\sum_{(k,i)\in E} f_{\ast ki}(x)$, while~\eqref{eq:Message:update:b}  follows from~\eqref{eq:HGAIHFIASGFA}.
Therefore, update equations~\eqref{eq:Message:update} compute quantities defined in~\eqref{eq:Message:definition}.
This also means that values $\widehat M_i(z)$ and $M_{ij}(z)$ are finite for all $z\in\mathbb R$ (due to the assumption made
in the beginning of Sec.~\ref{sec:intro}).

Now let us consider the backward pass. Rename nodes in $V$ as $\{1,\ldots,n\}$ (with $n$ being the root)
so that the procedure first assigns $x_n$, then $x_{n-1}$,  $x_{n-2}$, and so on until $x_1$. Note
that $V_i\subseteq\{1,\ldots,i\}$ for any $i\in V$. Define set
$$\calX_i=\{y\in\mathbb R^n\:|\:y_j=x_j\;\;\forall j\in\{i,\ldots,n\}\}$$
Next, we will prove that for any $i\in V$ we have $\min_{y\in\calX_i} f(y)\le\min_{y\in\mathbb R^n} f(y)$
(for $i=1$ this will mean that $f(x)\le\min_{y\in\mathbb R^n} f(y)$).
We use induction on $n$ in the decreasing order.
We have $x_n\in\argmin_{z} M_n(z)$ with $M_n(z)=\min_{x:x_i=z} f(x)$; this gives the base case $i=n$.
Now suppose the claim holds for $i+1\le n$; let us show it for $i$. Let $(i,j)\in E$ be the parent
edge for $i$, with $j\in\{i+1,\ldots,n\}$. 
We have $x_i=\pi_{ij}(x_j)$, or equivalently
\begin{equation}
x_i \in\argmin_{x_i}\left[\widehat M_i(x_i)+f_{ij}(x_j-x_i)\right] 
    = \argmin_{x_i}\left[\left(\min_{w:w_i=x_i} f_{\ast i}(w)\right)+f_{ij}(x_j-x_i)\right] \label{eq:fadjhasjdgbajdsg}
\end{equation}
Pick $x^\ast\in\argmin_{x^\ast\in\mathbb R^n:x^\ast_i=x_i} f_{\ast i}(x^\ast)$ and $y\in\argmin_{y\in\calX_{i+1}}f(y)$. 
Since function $f_{\ast i}(x^\ast)$ depends only on variables $x^\ast_j$ for $j\in V_i$,
other variables of $x^\ast$ can be chosen arbitrarily. We can thus assume w.l.o.g.\ that $x^\ast_j=y_j$ for all $j\in V-V_i$.
This implies that
$$
f(x^\ast)-f(y)=[f_{\ast i}(x^\ast)+f_{ij}(x_j-x_i)]-[f_{\ast i}(y)+f_{ij}(x_j-y_i)]
$$
%
(all other terms of $f$ cancel each other, and we have $x^\ast_j=y_j=x_j$ and $x^\ast_i=x_i$). We also have
\begin{eqnarray*}
 f_{\ast i}(x^\ast)+f_{ij}(x_j-x_i) \le  \left(\min_{w:w_i=y_i} f_{\ast i}(w)\right)+f_{ij}(x_j-y_i) \le  f_{\ast i}(y)+f_{ij}(x_j-y_i)
\end{eqnarray*}
where in the first inequality we used~\eqref{eq:fadjhasjdgbajdsg}. We obtain that $f(x^\ast)\le f(y)=\min_{y\in \calX_{i+1}}f(y)\le\min_{y\in \mathbb R^n}f(y)$,
where the last inequality is by the induction hypothesis. It can be checked that $x^\ast\in\calX_i$, which gives the claim for $i$.
\end{proof}

In the non-convex case we will use the DP algorithm directly.
In the convex case all
messages $\widehat M_i,M_{ij}$ will be convex functions, and it will be more convenient
to work with their derivatives $\widehat m_i,m_{ij}$, or more generally their subgradients if the messages
are not differentiable. (This will simplify considerably descriptions of algorithms).

The main computational question is how to manipulate with messages (or their subgradients). Storing their
values explicitly for all possible arguments is infeasible, so we need to use some implicit representation.
Specific cases are discussed below.


\section{Non-convex case}\label{sec:nonconvex}
In this section we assume that unary terms $f_i$ are continuous piecewise-linear functions with $O(1)$ breakpoints,
and terms $f_{ij}$ are given by~\eqref{eq:fij}.
As we will see, in this case all messages $\widehat M_i$ and $M_{ij}$ will also be continuous piecewise-linear functions.
We will store such function as a sequence $(s_0,\lambda_1,s_1,\ldots,s_{t-1},\lambda_t,s_t)$
where $t$ is the number of breakpoints, $\lambda_p$ is the $X$-coordinate of the $p$-th breakpoint with $\lambda_1\le\ldots\le\lambda_t$, and $s_p$ is the slope of the $p$-th segment.
Note that this sequence allows to reconstruct the message only up to an additive constant, but this will be sufficient for our purposes.

The two lemmas below address updates~\eqref{eq:Message:update:a} and~\eqref{eq:Message:update:b}, respectively;
their proofs are given in
Sec.~\ref{sec:lemma:generalUpdateSum} and~\ref{sec:lemma:generalUpdate}.
\begin{lemma}
If messages $M_{ki}$ for $(k,i)\in E$ are piecewise-linear functions with $t_k$ breakpoints
then $\widehat M_i=f_i+\sum_{(k,i)\in E}M_{ki}$ is also a piecewise-linear function with at most $t=O(1)+\sum_{(k,i)\in E}t_k$ breakpoints.
It can be computed in $O(t\log(d_i+1))$ time where $d_i=|\{k\:|\:(k,i)\in E\}|$ is the in-degree of node $i$.
\label{lemma:generalUpdateSum}
\end{lemma}

\begin{lemma}
If message $\widehat M_i$ is a piecewise-linear function with $t$ breakpoints then
 $M_{ij}=\widehat M_i\otimes f_{ij}$ is also a piecewise-linear function with at most $2t+1$ breakpoints. This min-convolution
and the corresponding mapping $\pi_{ij}$ can be computed in $O(t+1)$ time;
the latter is represented by a data structure of size $O(t+1)$ that can be queried in $O(t+1)$ time.

Furthermore, if $C_{ij}=+\infty$ then 
$M_{ij}$ has at most $t$ breakpoints.
\label{lemma:generalUpdate}
\end{lemma}

\begin{corollary}
If $C_{ij}=+\infty$ for all $(i,j)\in E$ then function~\eqref{eq:main} can be minimized using $O(n^2)$ time and space.
\label{corollary:nonconvex:L1}
\end{corollary}
\begin{proof}
The lemmas above imply that messages $\widehat M_i$ and $\widehat M_{ij}$ have at most $O(|V_i|)$ breakpoints,
where $V_i\subseteq V$ is the set of nodes in the subtree of $(V,E)$ rooted at $i$. Note that $|V_i|\le n$.
The overall time taken by updates~\eqref{eq:Message:update:b} is $\sum_{(i,j)\in E}O(n)=O(n^2)$. The same is true for updates~\eqref{eq:Message:update:a}, since
$$
\sum_{i\in V} n \log (d_i+1) \le const\cdot n \sum_{i\in V}  d_i = const\cdot n(n-1)
$$
Finally, we need $\sum_{(i,j)\in E}O(n)=O(n^2)$ space to store mappings $\pi_{ij}$, and $O(n^2)$ time to query them in the backward pass.
\end{proof}

If values $C_{ij}$ are finite then the lemmas
give only an exponential bound $2^{O(|V_i|)}$ on the number of breakpoints in messages $\widehat M_i$ and $\widehat M_{ij}$.
Our experiments, however, indicate that in practice the number of breakpoints stays manageable (see Sec.~\ref{sec:experiments}).


\subsection{Proof of Lemma~\ref{lemma:generalUpdateSum}}\label{sec:lemma:generalUpdateSum}
We only discuss the complexity of computing the sequence representing $\widehat M_i$; the rest of the statement is straightforward.
We need to compute a sum of $d_1+1$ piecewise-linear functions where
the breakpoints of each function are given in the non-decreasing order.
This essentially amounts to sorting all input breakpoints; clearly, during sorting we can also compute the slopes
between adjacent breakpoints. Sorting $d_i+1$ sorted lists with a total of $t$ points can be done in $O(t\log (d_i+1))$ time~\cite{Greene:93}.

\subsection{Proof of Lemma \ref{lemma:generalUpdate}}\label{sec:lemma:generalUpdate}
%
We will use the following fact, whose proof is given in Appendix~\ref{sec:ConsecutiveMinConvolution}.

\begin{proposition}
Suppose that  $g,g^1,\ldots,g^m$ are functions with $g^1(0)=\ldots=g^m(0)=0$ satisfying
\begin{equation}
g(z)=\min\{g^1(z),\ldots,g^m(z)\}\qquad \forall z\in\mathbb R
\label{eq:HGOAFHAGASF}
\end{equation}
and also suppose that $g$ satisfies a triangle inequality:
\begin{equation}
g(z_1+z_2)\le g(z_1)+g(z_2) \qquad \forall z_1,z_2
\label{eq:GHAOUSFGHAOUG}
\end{equation}
For a given function $h:\mathbb R\rightarrow\mathbb R$ define functions 
 $h^0,\ldots,h^m$ via $h^0=h$ and $h^k=h^{k-1}\otimes g^k$.
Then $h^m=h\otimes g$. Furthermore, if mappings $\pi^k$ correspond to min-convolutions $h^{k-1}\otimes g^k$
then mapping
$\pi=\pi^1\circ\ldots\circ\pi^m$ corresponds to min-convolution $h\otimes g$.
%
%
\label{prop:ConsecutiveMinConvolution}
\end{proposition}

Consider function $f_{ij}(z)=\min\{w|z|,C\}$ with $w\ge 0$.
We can assume w.l.o.g.\ that $C>0$ (otherwise computing $h\otimes f_{ij}=h+C$ is trivial).
It can be checked that $f_{ij}$ 
satisfies the triangle inequality. We will represent it as the minimum of the following three functions:
\begin{eqnarray*}
\hspace{-7pt}&f^1_{ij}(z)=\begin{cases}+\infty & \mbox{if }z<0 \\ wz &\mbox{if }z\ge 0\end{cases} \qquad
f^2_{ij}(z)=\begin{cases}w|z| & \mbox{if }z\le 0 \\ +\infty &\mbox{if }z> 0\end{cases}& \\
\hspace{-7pt}&f^3_{ij}(z)=\begin{cases}C & \mbox{if }z\ne 0 \\ 0 &\mbox{if }z= 0\end{cases}&
\end{eqnarray*}
If $C=+\infty$ then we can take just the first two functions. By
Proposition~\ref{prop:ConsecutiveMinConvolution}, it suffices to show 
 how to compute $h\otimes g$ and the corresponding mapping $\pi$ for a given piecewise-linear function $h$ and function $g\in\{f^1_{ij},f^2_{ij},f^3_{ij}\}$,
assuming the result in each case is also piecewise-linear. These transformations can then be applied consecutively to give 
$$
M_{ij}=\widehat M_i\otimes f_{ij}=((\widehat M_i\otimes f^1_{ij})\otimes f^2_{ij})\otimes f^3_{ij}
$$

\begin{figure*}[!t]
\center
\begin{tabular}{@{\hspace{0pt}}c@{\hspace{20pt}}c@{\hspace{12pt}}c@{\hspace{-0pt}}c@{\hspace{12pt}}c@{\hspace{-0pt}}c@{\hspace{12pt}}c}
(a) && (b) && (c) && (d) \vspace{-5pt}
\\
\includegraphics[height=70pt]{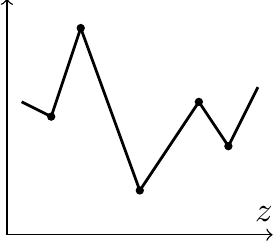}
& \raisebox{35pt}{$\mbox{\large$\Rightarrow$}$} &
\includegraphics[height=70pt]{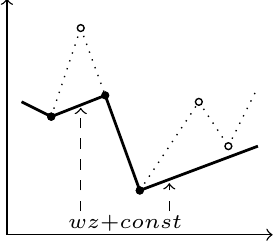}
& 
&
\includegraphics[height=70pt]{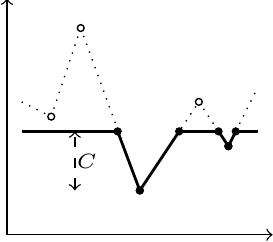}
& 
&
\includegraphics[height=70pt]{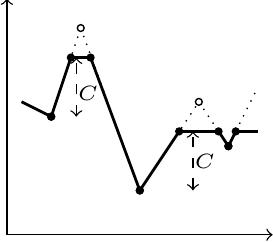}
\vspace{-2pt}
\\
 & & ~\scriptsize $g(z)=\begin{cases}+\infty & \mbox{if }z<0 \\ wz & \mbox{if }z\ge 0\end{cases}$ 
& & ~\scriptsize $g(z)=\begin{cases}C & \mbox{if }z\ne 0 \\ 0 & \mbox{if }z=0 \end{cases}$ 
& & ~\scriptsize $g(z)=\begin{cases}+\infty & \mbox{if }z<0 \\ 0 & \mbox{if }z=0 \\ C & \mbox{if }z> 0\end{cases}$ 
\end{tabular}
\caption{
(a): Input function $h(z)$. (b,c,d): Min-convolution $(h\otimes g)(z)$ for different functions $g$.
}
\label{fig:MinConvolution}
\end{figure*}

We assume below that $h$ is represented by the sequence $(s_0,\lambda_1,s_1,\ldots,s_{t-1},\lambda_t,s_t)$ with $t>0$. We also assume that $h$ bounded from below.
(This must be true for all messages, otherwise $f(x)$ would be unbounded from below).

\paragraph{Computing  $h\otimes f^1_{ij}$}
This operation is illustrated in Fig.~\ref{fig:MinConvolution}(b),
and a formal procedure for computing it is given in Algorithm~\ref{alg:DistanceTransform}.
Upon termination $\sigma$ is the sequence representing $h\otimes f^1_{ij}$
and $\Pi$ is a set of intervals defining mapping $\pi$ as follows:
$$
\pi(y)=\begin{cases}
\lambda^- & \mbox{if exists }[\lambda^-,\lambda^+]\in\Pi \mbox{ s.t.\ }y\in[\lambda^-,\lambda^+]\\
y & \mbox{otherwise}
\end{cases}
$$

\begin{algorithm}[!h]
\caption{Computing min-convolution $h\otimes f^1_{ij}$ 
}\label{alg:DistanceTransform}
\begin{algorithmic}[1]
\STATE set $p=1$, $\sigma=(s_0)$, $\Pi=\varnothing$
\STATE {\bf while} $p\le t$ {\bf do}
\STATE ~~~ append $(\lambda_p,\min\{s_p,w\})$ to $\sigma$
\STATE ~~~ {\bf if} $s_p\ge w$ {\bf then}
\STATE ~~~~~~ define linear function $\bar h(z)=w(z-\lambda_p)+h(\lambda_p)$
              passing through $(\lambda_p,h(\lambda_p))$ 
\STATE ~~~~~~ find smallest $q\!\in\![ p+2,t+1]$ s.t.\ $h(\lambda_{q})\!<\!\bar h(\lambda_{q})$, 
              \!assuming that $\lambda_{t+1}$ is sufficiently large;
     \\ ~~~~~~ if there is no such $q$ then add $[\lambda_p,+\infty]$ to $\Pi$ and terminate
\STATE ~~~~~~ compute $\lambda\in[\lambda_{q-1},\lambda_{q}]$  with $h(\lambda)=\bar h(\lambda)$
\STATE ~~~~~~ append $(\lambda,s_{q-1})$ to $\sigma$, add $[\lambda_p,\lambda]$ to $\Pi$, set $p:=q$
\STATE ~~~ {\bf else}
\STATE ~~~~~~ set $p:=p+1$
\STATE ~~~ {\bf end if}
\STATE {\bf end while}
\end{algorithmic}
\end{algorithm}

To verify correctness of Algorithm~\ref{alg:DistanceTransform},
first note that in line 6 we have $h(\lambda_{p+1})\ge \bar h(\lambda_{p+1})$
since $s_p \ge w$.
Therefore, in line 7 we are guaranteed to have $h(\lambda_{q-1})\ge \bar h(\lambda_{q-1})$
and $h(\lambda_{q})< \bar h(\lambda_{q})$,
and so value $\lambda$ in line 7 indeed exists.
It can also be checked that in line 3 we always have $s_{q-1}\le w$;
for $q=1$ this holds since $s_0\le 0$ due to the boundedness of $h$. \footnote{If we didn't have
the assumption that $f(x)$ is bounded, then we could modify Algorithm~\ref{alg:DistanceTransform}
as follows: if $s_0> w$ then return $\sigma=(w)$ and $\Pi=[-\infty,+\infty]$.}
This implies correctness of the procedure.

It can be seen that the number of breakpoints cannot increase:
if a new breakpoint $\lambda$ is introduced in line 8, then at least one
old breakpoint is removed, namely $\lambda_{q-1}$.


\paragraph{Computing  $h\otimes f^2_{ij}$} 
This case can be reduced to the previous one as follows:
$h\otimes f^2_{ij}=(h^{\tt rev}\otimes f^1_{ij})^{\tt rev}$ where $\varphi^{\tt rev}$ for a function $\varphi$ is defined via $\varphi^{\tt rev}(z)=\varphi(-z)$.
(To transform $\varphi$ to $\varphi^{\tt rev}$, we need to reverse the sequence for $\varphi$ and multiply all components by $-1$).

To reduce the number of passes through the sequence, one can modify Algorithm~\ref{alg:DistanceTransform}
so that it immediately produces the sequence for $h\otimes^{\tt rev} f^1_{ij}\eqdef (h\otimes f^1_{ij})^{\tt rev}$,
and then apply it twice noting that  $(h\otimes f^1_{ij})\otimes f^2_{ij}=(h\otimes^{\tt rev} f^1_{ij})\otimes^{\tt rev} f^1_{ij}$.

\paragraph{Computing  $h\otimes f^3_{ij}$} 
Assume that $0<C<+\infty$. Function $h$ is defined only up to an additive constant, so we can set e.g.\ $h(\lambda_1)=0$.
First,
we compute $p\in\argmin_{p\in[t]} h(\lambda_p)$
and set $C'=h(\lambda_p)+C$.
We now have $(h\otimes f^3_{ij})(z)=\min \{h(z),C'\}$ (Fig.~\ref{fig:MinConvolution}(c)).
Let $\Pi$ be the set of intervals whose union equals $\{z\:|\:h(z)\ge C'\}$,
then the mapping $\pi$ is given by
$$
\pi(y)=\begin{cases}
\lambda_q & \mbox{if $y\in[\lambda^-,\lambda^+]$ for some $[\lambda^-,\lambda^+]\in\Pi$} \\
y & \mbox{otherwise}
\end{cases}
$$
The number of breakpoints in $h\otimes f^3_{ij}$ is at most $2t+1$ since for each of the $t+1$ linear segments of $h$ at most one new breakpoint is introduced.

Clearly, the sequence for $h\otimes f^3_{ij}$ and the set $\Pi$ can be computed in $O(t+1)$ time; we omit details.


\subsection{Extensions}
To conclude the discussion of the non-convex case, we mention two possible extensions: 
\begin{itemize}
\item[(i)] Allow pairwise terms $f_{ij}$ to be piecewise-linear functions that are 
non-increasing concave on $(-\infty,0]$ and non-decreasing concave on $[0,+\infty)$. 
(A truncated TV regularizer is a special case of that.)
\item[(ii)] Allow unary terms $f_i$ to be piecewise-quadratic.
\end{itemize}
We claim that in both cases messages can be computed exactly
(either as piecewise-linear or piecewise-quadratic), although the number of breakpoints could grow exponentially.
Below we give a proof sketch only for the first extension, in which the messages stay piecewise-linear.

Adding a constant to $f_{ij}$ does not change the problem, so we can assume w.l.o.g.\ that $f_{ij}(0)=0$.
If $f_{ij}$ has $m-1$ breakpoints then we can represent it as a minimum of functions $f^1_{ij},\ldots,f^m_{ij}$
where $f_{ij}^k$ satisfies $f_{ij}^k(0)=0$ and is
either (a) linear on $(-\infty,0)$ and $+\infty$ on $(0,+\infty)$,
or (b) vice versa: $+\infty$ on $(-\infty,0)$ and linear on $(0,+\infty)$.
It can be checked that $f_{ij}$ satisfies the triangle inequality,
so we can apply Proposition~\ref{prop:ConsecutiveMinConvolution}.
It is thus sufficient to describe how to compute min-convolution $h\otimes g$ for each $g\in\{f^1_{ij},\ldots,f^m_{ij}\}$.

Assume that $g$ is $+\infty$ on $(-\infty,0)$ and linear on $(0,+\infty)$
 (the other case is symmetric).
We have $g(z)=\alpha z+C$ for $z>0$, where $\alpha,C\ge 0$ are some constants. 
For a function $\varphi$ define function $\varphi^\alpha$ via $\varphi^\alpha(z)=\varphi(z)-\alpha z$.
(Note that adding a linear term to a $\varphi$  can be done 
by traversing the sequence representing $\varphi$ and increasing all slopes by a constant.)
It can be checked that $h\otimes g=(h^\alpha\otimes g^\alpha)^{-\alpha}$,
so it suffices to consider the min-convolution $h^\alpha\otimes g^\alpha$.
Such min-convolution is illustrated in Fig.~\ref{fig:MinConvolution}(d).
It is not difficult to see that it adds at most $t$ breakpoints
and can be implemented in $O(t+1)$ time, where $t$ is the number of breakpoints of $h$.
We leave details to the reader.


\section{Convex case}\label{sec:convex}
We now assume that all functions $f_i$ and $f_{ij}$ are convex; as we will see, in this case
function $f(x)$ can be minimized much more efficiently. For convenience of notation we will assume
that functions $f_{ij}$ are given by
$$
f_{ij}(z)=\begin{cases}
w^-_{ij}\cdot z & \mbox{if }z < 0 \\
w^+_{ij}\cdot z & \mbox{if }z \ge 0 
\end{cases}
$$
with $w^-_{ij}\le w^+_{ij}$.

\begin{figure}[!t]
\center
\begin{tabular}{@{\hspace{0pt}}c@{\hspace{12pt}}c@{\hspace{12pt}}c@{\hspace{12pt}}c@{\hspace{12pt}}c}
{\small $\widehat M_{i}(z)$} & &
{\small ~~~$M_{ij}(z)$}
\vspace{-15pt} 
\\
\includegraphics[scale=0.9]{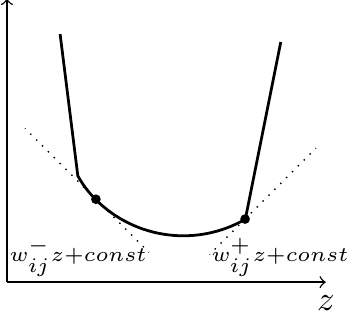}
& \raisebox{40pt}{$\mbox{\large$\Rightarrow$}$} &
\raisebox{-8.5pt}{
\includegraphics[scale=0.9]{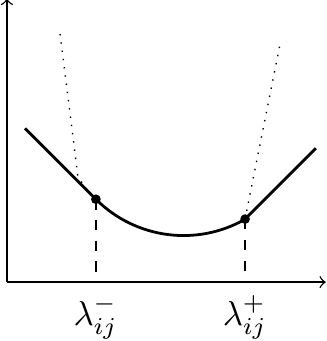}
}
\end{tabular}
\caption{Min-convolution $M_{ij}=\widehat M_i\otimes f_{ij}$. If $\widehat M_i$ is convex then
so is $M_{ij}$. Function $M_{ij}$ coincides with $\widehat M_i$
on $[\lambda^-_{ij},\lambda^+_{ij}]$ and is linear on $(-\infty,\lambda^-_{ij}]$ and $[\lambda^+_{ij},+\infty)$ 
with the slopes $w^-_{ij}$ and $w^+_{ij}$ respectively.
}
\label{fig:ConvexMessage}
\end{figure}

It can be checked that if function $\widehat M_i$ is convex then so is $M_{ij}=\widehat M_i\otimes f_{ij}$ (see Fig.~\ref{fig:ConvexMessage}),
and therefore 
%
%
by induction all messages $\widehat M_i$ and $M_{ij}$ will be convex.
It will be more convenient to work with their derivatives $\widehat m_i(z)={\widehat M}'_i(z)$
and $m_{ij}(z)=M'_{ij}(z)$. If $\widehat M_i$ is not differentiable at $z$ then
we let $\widehat m_i(z)$ to be an arbitrary subgradient of $M_i$ at $z$ (and similarly for $M_{ij}$).
Note that functions $\widehat m_i$, $m_{ij}$ are non-decreasing and satisfy
\begin{eqnarray*}
\widehat M_{i}(z)&=&const+\int_{0}^{z}\widehat m_{i}(\lambda)d\lambda \qquad \hspace{3pt} \forall  z\in\mathbb R \\
M_{ij}(z)&=&const+\int_{0}^{z}m_{ij}(\lambda)d\lambda \qquad \forall z\in\mathbb R
\end{eqnarray*}

Let $g_i(z)$ be a subgradient of $f_i$ at $z$; function $g_i$ is then non-decreasing (and not necessarily continuous).
An algorithm that works with subgradients is given below (Algorithm~\ref{alg:DP}). 
In this algorithm we denote  ${\tt clip}_{[a,b]}(z)=\min\{\max\{z,a\},b\}\}$ (i.e.\ the projection of $z$ to $[a,b]$). The updates in lines 2 and 3 correspond
to eq.~\eqref{eq:Message:update:a} and~\eqref{eq:Message:update:b} respectively,
and the values $\lambda^-_{ij},\lambda^+_{ij}$ in line 4 describe the mapping $\pi_{ij}$ corresponding to the min-convolution $M_{ij}=\widehat M_i\otimes f_{ij}$.

\begin{algorithm}[!h]
\caption{DP algorithm for the convex case. 
}\label{alg:DP}
\begin{algorithmic}[1]
\setalglineno{0}
\STATE add new node $\hat r$ and edge $(r,\hat r)$ to $(V,E)$, make $\hat r$ the new root;
set $w^-_{r\hat r}=w^+_{r\hat r}=0$ 
\\ // {\em forward pass; steps 2 and 3 should be done for all $z\in\mathbb R$}\!\!\!
\STATE {\bf for each} edge $(i,j)\in E$ {\bf do} in the order from the leaves toward the root
\STATE ~~~ set
$
\widehat m_{i}(z)=g_i(z)+\mathlarger\sum\limits_{(k,i)\in E} m_{ki}(z)
$
\STATE ~~~  set
$
m_{ij}(z)={\tt clip}_{[w^-_{ij},w^+_{ij}]}(\widehat m_{i}(z))
$
\STATE ~~~ find interval $[\lambda^-_{ij}$, $\lambda^+_{ij}]$ such that
     $\widehat m_i(\lambda^-_{ij})=w^-_{ij}$
and $\widehat m_i(\lambda^+_{ij})= w^+_{ij}$ 
\STATE {\bf end for}
\\ // {\em backward pass}
\STATE set $x_{r}\in[\lambda^-_{r\hat r},\lambda^+_{r\hat r}]$
\STATE {\bf for each} edge $(i,j)\in E-\{(r,\hat r)\}$ {\bf do} in the order from the root toward the leaves
\STATE
~~~ set $x_i={\tt clip}_{[\lambda^-_{ij},\lambda^+_{ij}]}(x_j)$
\STATE {\bf end for}
\end{algorithmic}
\end{algorithm}

If in line 4 there is no $\lambda^-_{ij}$ with $\widehat m_i(\lambda^-_{ij})=w^-_{ij}$ then
we use a natural rule (considering that function $\widehat m_i$ is non-decreasing), namely
set 
\begin{equation}
\lambda^-_{ij}\in[
\;\sup\{z\:|\:\widehat m_i(z)<w^-_{ij}\},\;
\inf\{z\:|\:\widehat m_i(z)>w^-_{ij}\}\;
]
\label{eq:FALSIGAHSG}
\end{equation}
Note that the bounds in~\eqref{eq:FALSIGAHSG} (and thus $\lambda^-_{ij}$)
can be infinite.
A similar rule is used for $\lambda^+_{ij}$.

Note that Algorithm~\ref{alg:DP} is equivalent to the one given by Johnson~\cite{Johnson2013},
except that the latter has been formulated for smooth functions $f_i$ and a chain graph $(V,E)$.

\subsection{Quadratic unaries on a chain: $O(n)$ algorithm}\label{sec:Convex:Quadratic}
Let us assume that graph $(V,E)$ is a chain 
with nodes $V=[n]$, where $n$ is the root.
We also assume that the unary terms are  strictly convex quadratic functions:
$f_i(x_i)=\frac{1}{2}a_i x_i^2 - b_i x_i$ with $a_i>0$.
We thus have $g_i(x_i)=a_i x - b_i$.
As shown in~\cite{Johnson2013}, Algorithm~\ref{alg:DP} can be implemented in $O(n)$ time.
In this section we describe a more memory-efficient version:
we use 2 floating points per breakpoint, while~\cite{Johnson2013} used 3 floating points plus
a Boolean flag.

It can be  checked by induction that all messages $m_{ij}$ and $\widehat m_i$ are
piecewise-linear non-decreasing functions, and the latter are strictly increasing (see Fig.~\ref{fig:MessageUpdate}).
We will maintain the current message  (which can be either $\widehat m_i$ or $m_{ij}$)  with $t$ breakpoints and $t+1$ segments
as a sequence $(s_0,\lambda_1,s_1,\ldots,s_{t-1},\lambda_t,s_t)$.
Here $\lambda_p$ is the $X$-coordinate of the $p$-th breakpoint with $\lambda_1\le\ldots\le\lambda_t$, and $s_p$ represents in a certain way the slope of the $p$-th segment.

If $s_p$ were the true slope of the corresponding segment then in transformation (a)
in Fig.~\ref{fig:MessageUpdate} we would need to go through the entire sequence and increase the slopes
by $a_i$.
To avoid this expensive operation, we use an implicit representation:
the true slope of the $p$-th segment in messages $\widehat m_i$ and $ m_{ij}$ is given by $s_p+\overline a_i$,
where $\overline a_i=\sum_{k=1}^i a_k$. Thus, the transformation (a) in Fig.~\ref{fig:MessageUpdate} 
is performed automatically; we just need to compute $\overline a_i=\overline a_{i-1} +a_i$.

Sequence $(s_0,\lambda_1,s_1,\ldots,s_{t-1},\lambda_t,s_t)$ will be stored contiguously in an array of size $4n+O(1)$
together with indexes pointing to the first and the last elements.
In the beginning the sequence is placed in the middle of this array, so that there is a sufficient space for growing both to the left and to the right.

\begin{figure*}[!t]
\center
\begin{tabular}{@{\hspace{0pt}}c@{\hspace{12pt}}c@{\hspace{12pt}}c@{\hspace{12pt}}c@{\hspace{12pt}}c}
%
%
{\small ~~~~~~~$m_{ki}(z)$} & &
{\small ~~~$\widehat m_i(z)$} & &
{\small ~~~~~~~$m_{ij}(z)$}  
\vspace{-10pt} 
\\
\includegraphics[height=70pt]{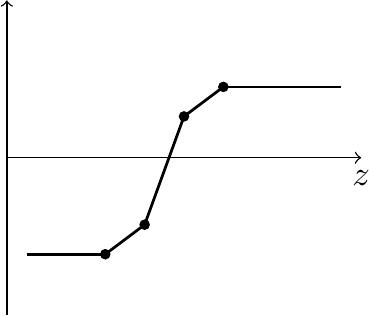}
& \raisebox{32pt}{$\stackrel{\mbox{\raisebox{1pt}{\small(a)}}}{\mbox{\large$\Rightarrow$}}$} &
\includegraphics[height=70pt]{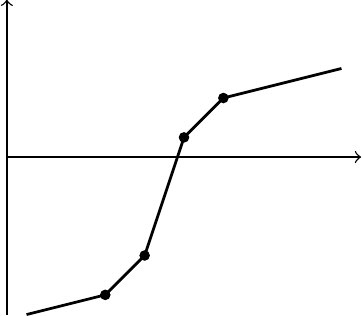}
& \raisebox{32pt}{$\stackrel{\mbox{\raisebox{1pt}{\small(b)}}}{\mbox{\large$\Rightarrow$}}$} &
\includegraphics[height=70pt]{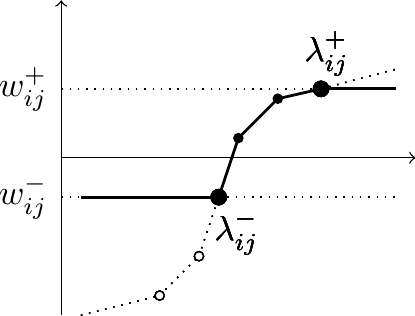}
\end{tabular}
\caption{TV on a chain with quadratic unaries; $(k,i,j)=(i\!-\!1,i,i\!+\!1)$. Message $m_{ij}$ is obtained from $m_{ki}$
via two transformations:
$\widehat m_{i}(z)=m_{ki}(z)+a_i\lambda-b_i$
and
$m_{ij}(z)={\tt clip}_{[w^-_{ij},w^+_{ij}]}(\widehat m_i(z))$.
In transformation (b) two new breakpoints are added (shown as thick black circles),
and some breakpoints may be removed (shown as empty circles).
}
\label{fig:MessageUpdate}
\end{figure*}

\paragraph{Forward pass} We have described the data structure used for storing the messages;
let us now discuss how these data structures are updated during transformation (b) in Fig.~\ref{fig:MessageUpdate} for edge $(i,j)\in E$.
We assume that $(s_0,\lambda_1,s_1,\ldots,s_{t-1},\lambda_t,s_t)$, $t\ge 2$
is the current sequence for message $\widehat m_i$. 

During the update some breakpoints at the two ends of the sequence may be removed
(they are shown as empty circles in Fig.~\ref{fig:MessageUpdate}),
and two new breakpoints $\lambda^-_{ij}$ and $\lambda^+_{ij}$ are appended at the ends.
Let $d_i\ge 0$ be the number of removed breakpoints.
We show below that the update can be performed in $O(d_i+1)$ time.
This will imply that the forward pass takes $O(n)$ time. Indeed, we can use an amortized analysis;
when adding a new breakpoint, we give it one unit of credit, and use this credit when the breakpoint
is removed. The total number of added breakpoints is $2n$, which gives the claim.

In the remainder of this section we describe details of the update for edge $(i,j)$.
To determine where $\lambda^-_{ij}$ is to be inserted, we first need to find the largest index $\ell$ with $\widehat m_i(\lambda_{\ell})<w^-_{ij}$;
if there is no such index then let $\ell=0$. This can be done by traversing the breakpoints from left to right, 
stopping when index $\ell$ is found.
Note that values $\widehat m_i(\lambda_k)$ are not stored explicitly, but can be recursively computed as follows (here $k=i-1$):
\begin{eqnarray*}
\widehat m_i(\lambda_1)&=&w^-_{ki}+a_i\lambda_1-b_i \\
\widehat m_i(\lambda_{p+1})&=&\widehat m_i(\lambda_p) + (s_p+\overline a_i)\cdot(\lambda_{p+1}-\lambda_p)
\end{eqnarray*}
Once $\ell$ is computed, we  can find the value $\lambda^-_{ij}$ for which
 $\widehat m_i(\lambda^-_{ij})=w^-_{ij}$ in $O(1)$ time.
We then change the sequence to
$(-\overline a_i,\lambda^-_{ij},s_{\ell+1},\ldots,s_{t-1},\lambda_t,s_t)$.
All these operations
take $O(\ell+1)$ time.

Inserting breakpoint $\lambda^+_{ij}$ is performed in a similar way. We traverse the sequence from right to left
and compute the smallest index $r\ge \ell+1$ with $\widehat m_i(\lambda_r)>w^+_{ij}$; if there is no such $r$ then let $r=t+1$.
We use recursions
\begin{eqnarray*}
\widehat m_i(\alpha_t)&=&w^+_{ki}+a_i\lambda_{t}-b_i \\
\widehat m_i(\lambda_p)&=&\widehat m_{i}(\lambda_{p+1}) - (s_p+\overline a_i)\cdot(\lambda_{p+1}-\lambda_p)
\end{eqnarray*}
We then set $\lambda^+_{ij}$ so that $\widehat m_i(\lambda^+_{ij})=w^+_{ij}$
and change the sequence to
$(-\overline a_i,\lambda^-_{ij},s_{\ell+1},\ldots,s_{r-1},\lambda^+_{ij},-\overline a_i)$.

\paragraph{Backward pass}
In the backward pass we need to update $x_i={\tt clip}_{[\lambda^-_{ij},\lambda^+_{ij}]}(x_j)$ for all edges $(i,j)\in E$
in the backward order.
These updates can be performed in $O(n)$ time if
we record values $\lambda^-_{ij}$ and  $\lambda^+_{ij}$
during the forward pass.

The overall memory requirements of the algorithm (excluding input parameters) is thus $6n+O(1)$ floating point numbers:
$4n+O(1)$ for storing the sequence $(s_0,\lambda_1,s_1,\ldots,s_{t-1},\lambda_t,s_t)$
and $2n+O(1)$ for storing breakpoints $\lambda^-_{ij},\lambda^+_{ij}$ for all edges $(i,j)$.
Note, breakpoints $\lambda^-_{ij}$ can be stored in the same array used for returning the solution $x$.

\subsection{Piecewise-quadratic unaries: $O(n\log n)$ algorithm}\label{sec:Convex:PiecewiseQuadratic}
To simplify the presentation, we will first consider the case of piecewise-linear unaries on a chain,
and then discuss extensions to trees and to piecewise-quadratic unaries.

\paragraph{Piecewise-linear unaries on a chain}
In this case terms $g_i=f'_i$ and the messages $m_{ij}$ and $\widehat m_i$ will be piecewise-constant non-decreasing functions;
Fig.~\ref{fig:PiewiseLinear} illustrates how they are updated. The current message $h:\mathbb R\rightarrow\mathbb R$ (which is either $m_{ij}$ or $\widehat m_i$)
will be represented by the following data: (i) values $h^-=h(-\infty)=\min_z h(z)$ and $h^+=h(+\infty)=\max_z h(z)$;
(ii) a multiset $S$ of breakpoints of the form $\sigma=(\lambda_\sigma,\delta_\sigma)$ 
where  $\lambda_\sigma$ is its $X$-coordinate and $\delta_\sigma$ is the increment in the value of $h$ at this breakpoint.
We thus have
$$
h(z)=h^- + \sum_{\sigma\in S:\lambda_\sigma<z} \delta_\sigma = h^+ - \sum_{\sigma\in S:\lambda_\sigma\ge z} \delta_\sigma
$$
assuming that $g$ is left-continuous.
Points $\sigma\in S$ will be stored in a {\em double ended priority queue} which
allows the following operations: 
 {\tt Insert} (which inserts a single a point into $S$),
 {\tt FindMin} (which finds a point $\sigma\in S$ with the minimum value of $\lambda_\sigma$),
 {\tt FindMax},
 {\tt RemoveMin}, and
 {\tt RemoveMax}.
In our implementation we use a Min-Max Heap~\cite{MinMaxHeap} which takes $O(1)$ for {\tt FindMin}/{\tt FindMax}
and $O(\log n)$ for {\tt Insert}/{\tt RemoveMin}/{\tt RemoveMax} (assuming that the total number of points is bounded by $O(n)$).

\begin{figure*}[!t]
\center
\begin{tabular}{@{\hspace{0pt}}c@{\hspace{12pt}}c@{\hspace{12pt}}c@{\hspace{12pt}}c@{\hspace{12pt}}c}
%
%
{\small ~~~$m_{ki}(z)$} & &
{\small ~~~~~~~$\widehat m_{i}(z)$}  & &
{\small ~~~~~~~$m_{ij}(z)$} 
\vspace{-10pt} 
\\
\includegraphics[height=70pt]{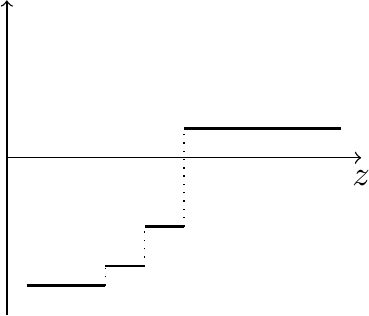}
& \raisebox{32pt}{$\stackrel{\mbox{\raisebox{1pt}{\small(a)}}}{\mbox{\large$\Rightarrow$}}$} &
\includegraphics[height=70pt]{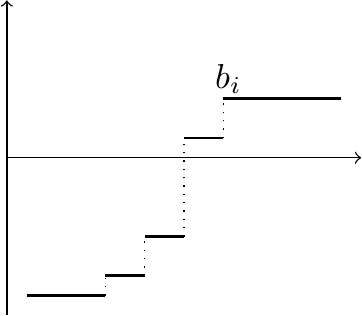}
& \raisebox{32pt}{$\stackrel{\mbox{\raisebox{1pt}{\small(b)}}}{\mbox{\large$\Rightarrow$}}$} &
\includegraphics[height=70pt]{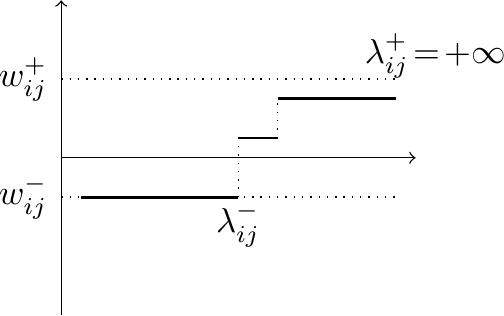}
\end{tabular}
\caption{TV on a chain with piecewise-linear unaries.
}
\label{fig:PiewiseLinear}
\end{figure*}
Let us discuss how to update this data during message passing. First, consider the update $\widehat m_{ij}=m_{ki}+g_i$.
If $g_i$ has one breakpoint, i.e.\ 
$g_i(z)=\begin{cases}a^-_i & \mbox{if } z\le b_i \\ a^+_i & \mbox{if } z> b_i \end{cases}$,
then we insert $\sigma=(b_i,a^+_i-a^-_i)$ into $S$ and update $h^-:=h^-+a^-$, $h^+:=h^++a^+$.
If $g_i$ has more than one breakpoint then the procedure is similar. Now consider the update $m_{ij}(z)={\tt clip}_{[w^-_{ij},w^+_{ij}]}(\widehat m_{i}(z))$.
To clip from below, we first need to remove points $\sigma\in S$ with $\widehat m_{i}(\lambda_\sigma+0)\le w^-_{ij}$.
For that we repeatedly call {\tt FindMin}/{\tt RemoveMin}  (updating $h^-$ accordingly) until we get $h^-\ge w^-_{ij}$.
As the last step, we may need to update 
the value $\delta_\sigma$ for $\sigma={\tt FindMin}$. (The case when $S$ becomes empty should be handled separately.) Clipping from above is done in a similar way.
During this procedure we can also compute values $\lambda^-_{ij}$ and  $\lambda^+_{ij}$.

It remains to discuss the complexity. {\tt Insert}, {\tt RemoveMin} and {\tt RemoveMax} operations are called at most $O(n)$ times
(the number of points is $O(n)$), so they take $O(n\log n)$ time. When computing $m_{ij}$, the number of calls to {\tt FindMin}
exceeds that to {\tt RemoveMin} by at most 1 (and similarly for ``{\tt Max}''), so {\tt FindMin}/{\tt FindMax} are called $O(n)$ times
and take $O(n)$ time.
\paragraph{Extension to trees} If graph $(V,E)$ is a tree, we will use the same data structure for each branch
as we go from the leaves toward the root. The difference from the previous case
is that now during the update $\widehat m_i=g_i+\sum_{(k,i)\in E}m_{ki}$ we need to compute the union of multisets
corresponding to messages $m_{ki}$. Thus, we need a version of double ended priority queue that allows
merging two queues. One possibility is to use two Fibonacci Heaps~\cite{FibonacciHeap} (one for the $\min$ and one for the $\max$ operations)
which allow merging in $O(1)$. The total number of merge operations is $O(n)$, so the overall complexity is still $O(n\log n)$.

\paragraph{Extension to piecewise-quadratic unaries}
In this case terms $g_i$ (and thus messages $m_{ij}$, $\widehat m_i$)
will be non-decreasing piecewise-linear functions 
(possibly, discontinuous at breakpoints).
We will use the same approach as before, only now each segment will be specified via two numbers $(a,b)$
(that define function $az+b$), not one.
Thus, $h^-$ is now a vector $h^-=(a^-,b^-)$ with $h(z)=a^- z + b^-$
for $z\rightarrow-\infty$, and similarly for $h^+$. A breakpoint $\sigma$ is now given by a triple $\sigma=(\lambda_\sigma,
\delta a_\sigma,\delta b_\sigma)$ where the last two values describe the change in the parameters of the linear function
at this breakpoint. One difference in the algorithm 
is that now new breakpoints may appear during the update $m_{ij}(z)={\tt clip}_{[w^-_{ij},w^+_{ij}]}(\widehat m_{i}(z))$,
similar to the case in Sec.~\ref{sec:Convex:Quadratic}. However, the number of such breakpoints is at most 2,
and therefore the complexity is still $O(n\log n)$.

\subsection{Piecewise-linear unaries on a chain: $O(n\log\log n)$ algorithm}\label{sec:Convex:PiecewiseLinear}
Here we assume again that $(V,E)$ is a chain with $V=[n]$ and
functions $f_i$ are piecewise-linear with $O(1)$ breakpoints. Thus, functions $g_i$ are non-decreasing piecewise-constant; we will assume that they are left-continuous.
It is well-known~\cite{Murota:book} that any submodular function $h(z)$ has a unique {\em lowest} minimizer; we will denote it as $\argmin^-_z h(z)\in\argmin_z h(z)$.
We will show how to compute $x=\argmin^-_x f(x)$. To simplify the presentation,
we will assume that it is bounded, i.e.\ $x\in\mathbb R^n$. 

Let $\Lambda$ be the multiset of breakpoint values $\lambda$ present in unary terms, so that $|\Lambda|=O(n)$.
It can be checked 
that there exists an optimal solution $x\in\Lambda^n$ (e.g.\ by observing that the
algorithm given in the previous section never introduces new breakpoint values).

Note that in all previous algorithms we explicitly computed values $\lambda^-_{ij}$ and $\lambda^+_{ij}$.
The proposition below shows that we cannot afford to do this anymore if we want to improve on $O(n\log n)$.
(Its proof is given in Appendix~\ref{sec:SortingLowerBound}, and is based on a reduction to the sorting problem.)
\begin{proposition}
Any comparison-based algorithm that computes values $\lambda^+_{ij}$ for all $(i,j)\in E$ 
requires $\Omega(n\log n)$ comparisons in the worst case.
\label{prop:SortingLowerBound}
\end{proposition}
To motivate our approach, we will first describe an alternative $O(n\log n)$ algorithm 
that avoids computing values  $\lambda^-_{ij}$ and $\lambda^+_{ij}$. This algorithm can be viewed as a specialization of Hochbaum's algorithm~\cite{Hochbaum:ACM01} to chain graphs.
We will then show how to modify it to get $O(n\log\log n)$ complexity.


The idea is to 
reduce minimization problem~\eqref{eq:main} 
to a sequence of problems of the following form (for some fixed values of parameter $\lambda\in\mathbb R$): 
\begin{equation*}
\min_{y\in\{0,1\}^n} g_\lambda(y),\quad g_\lambda(y)=\sum_{i\in V} g_i(\lambda)y_i + \sum_{(i,j)\in E} f_{ij}(y_j-y_i)
\end{equation*}
Such reduction to a {\em parametric maxflow problem}, due to Hochbaum~\cite{Hochbaum:ACM01}, is well-known for the TV
problem on general graphs; it also appeared later in~\cite{Zalesky:JAM02,Darbon:JMIV:I,Chambolle:EMMCVPR05}.
It is based on  the following result. 
\begin{theorem}[\cite{Hochbaum:ACM01}]
For a fixed $\lambda\in\mathbb R$, let $y=\argmin^-_{y\in\{0,1\}^n}g_\lambda(y)$.
Denote $V_0=\{i\in V\:|\:y_i=0\}$ and
$V_1=\{i\in V\:|\:y_i=1\}$. 
Let $x=\argmin^-_{x\in\mathbb R^n}f(x)$; for brevity write 
 $x=(x^0,x^1)$ 
where $x^k$ is the subvector of $x$ corresponding to subset $V_k$.
Then
$x^0< \overline\lambda$ and $x^1\ge \overline\lambda$ component-wise,
where $\overline \lambda$  denotes vector $(\lambda,\ldots,\lambda)$ of
the appropriate dimension.
\end{theorem}
The theorem suggests a divide-and-conquer algorithm  for computing $x=\argmin^-_{x\in[a,b]^n} f(x)$:
\begin{itemize}
\item Pick some ``pivot'' value $\lambda\in[a,b]$ and compute $y=\argmin^-_{y\in\{0,1\}^n}g_\lambda(y)$;
this partitions the nodes into two subsets $V_0$ and
$V_1$ of sizes $n_0$ and $n_1$ respectively.
\item Compute recursively $x^0=\argmin^-_{x^0\in[a,\lambda]^{n_0}} f(x^0,\overline\lambda)$
and $x^1=\argmin^-_{x^1\in[\lambda,b]^{n_1}} f(\overline\lambda,x^1)$ (or solve these problems explicitly,
if e.g.\ their size is small enough).\footnote{Note that for any fixed $x^1\ge\overline\lambda$
we have $f(x^0,x^1)=f(x^0,\overline\lambda)+const$ for $x^0\le\overline\lambda$.
This justifies replacing the objective function $f(x^0,x^1)$ with $f(x^0,\overline\lambda)$
in the first subproblem. A similar argument
holds for the second subproblem.}
These two subproblems are defined on induced subgraphs $(V_0,E[V_0])$ and $(V_1,E[V_1])$, respectively.
Each of them is a union of chains; each chain is solved independently via a recursive call.
\end{itemize}

Let us apply this strategy to our problem. First, observe that for a fixed $\lambda$ function $g_\lambda(y)$
can be minimized in $O(n)$ time.
Indeed, we can use a dynamic programming approach described in Sec.~\ref{sec:background},
except that instead of continuous-valued variables we now have $\{0,1\}$-valued
variables. Let 
 $\widehat m_j(y_j;\lambda)$ and $m_{ij}(y_j;\lambda)$ be the corresponding messages where $j\in V$,   $(i,j)\in E$
and $y_j\in\{0,1\}$. To extract an optimal solution, it suffices to know the differences $\widehat m_i(1;\lambda)-\widehat m_i(0;\lambda)$
and $m_{ij}(1;\lambda)-m_{ij}(0;\lambda)$.
Denote these differences as $\widehat m_i(\lambda)$ and $m_{ij}(\lambda)$ respectively.
It can be checked that the update equations for these values are given by lines 2 and 3 of Algorithm~\ref{alg:DP}
for $z=\lambda$, and each of these updates takes $O(1)$ time. 

Second, we note that in the subproblem $\min_{x^0\in [a,\lambda]^{n_0}} f(x^0,\overline\lambda)$
we can modify the unary terms for nodes $i\in V_0$ by removing all breakpoints that are greater than or equal to $\lambda$;
this will not change the problem. Similarly, for the other subproblem we can remove all
breakpoints that are smaller than or equal to $\lambda$.

Finally, we need to discuss how to select value $\lambda$. It is natural to take $\lambda$ as 
the median of values in $\Lambda$, which can be computed in $O(|\Lambda|)\subseteq O(n)$ time~\cite{Cormen}.
(If $\Lambda$ is empty, then we can solve the problem explicitly by taking $x=\overline a$.)

Let $\Lambda_0$ and $\Lambda_1$ be the multisets of breakpoints present in the first and the second
subproblems respectively. We have $|\Lambda_0|\le \frac{1}{2}|\Lambda|$ and  $|\Lambda_1|\le \frac{1}{2}|\Lambda|$,
which leads to the following complexity (see Appendix \ref{sec:proof:nlogn}).
\begin{proposition}
The algorithm above has complexity $O(n\log n)$.
\label{prop:nlogn}
\end{proposition}
We now discuss how to modify this approach to get $O(n\log \log n)$ complexity.
Choose an integer $m\in \Theta(\log n)$,
and define set $\calU=\{km+1\in [n]\:|\:k\in\mathbb Z\}\cup\{n\}$ of size $N=|\calU|\in\Theta(n/\log n)$.
The nodes in $\calU$ will be called {\em subsampled nodes}.
We assume that $\calU=\{i_1,\ldots,i_N\}$ with $1=i_1<\ldots<i_N=n$,
and let $\calE=\{(i_k,i_{k+1})\:|\:k\in[N-1]\}$. 

The algorithm will have two stages.
First, we use a divide-and-conquer strategy above to compute an optimal solution $x_i$
for nodes $i\in \calU$. Once this is done, a full optimal solution can be recovered 
by solving $|\calE|$ independent subproblems: for each $(i,j)\in\calE$
we need to minimize function $f(x)$ over $(x_{i+1},\ldots,x_{j-1})$ 
with fixed values of $x_i$ and $x_j$. The latter can be done in $O(m\log m)$
time (since $j-i\le m$), so the complexity of the second stage is $O(Nm\log m)=O(n\log\log n)$.

We thus concentrate on the first stage. Its main computational subroutine
is to compute an optimal solution $y=\argmin^-_{y\in\{0,1\}^{n}} g_\lambda(y)$
for a given $\lambda$
at nodes $i\in \calU$. As before, we will use dynamic programming.
Passing messages in a naive way would take
$O(Nm)$ time, which is too slow for our purposes. To speed it up,
we will ``contract'' each edge $(i,j)\in\calE$ into a data structure
that will allow passing a message from $i$ to $j$ in $O(\log m)$ time instead of $O(m)$, 
so that the subroutine will take $O(N\log m)$ time. The contraction operation is described below;
we then give a formal description of the first stage.

\paragraph{Contraction} Consider indices $i,j\in V$ with $i<j$.
Our goal to solve efficiently the following problem for a given $\lambda\in\mathbb R$: given the value $\widehat m_{i}(\lambda)$,
compute message $\widehat m_{j}(\lambda)$. 
Let us denote the corresponding transformation by $T^\lambda_{ij}:\mathbb R\rightarrow\mathbb R$, so that
$\widehat m_{j}(\lambda)=T^\lambda_{ij}(\widehat m_{i}(\lambda))$.  
We will show that mapping $T^\lambda_{ij}$ can be described compactly by 3 numbers.\!\!\!\!\!\!\!\!
\begin{proposition} For a triplet  $\tau=(\delta,a,b)\in\mathbb R^3$ with $a\le b$ 
define  function $\langle\tau\rangle:\mathbb R\rightarrow\mathbb R$ via
\begin{equation*}
\langle\tau\rangle(v)=\delta + {\tt clip}_{[a,b]}(v)\qquad\forall v\in\mathbb R
\end{equation*}
(a) If $(i,j)\in E$ then $T^\lambda_{ij}=\langle g_j(\lambda),w^-_{ij},w^+_{ij}\rangle$. (b) There holds
%
%
%
%
\begin{equation*}
\langle\delta',a',b'\rangle\circ\langle\delta,a,b\rangle \; = \;
\begin{cases}
\langle \delta' + a' - b,b,b\rangle & \mbox{if } b<\min I \\
\langle \delta'+\delta,{\tt clip}_I(a),{\tt clip}_I(b)\rangle & \mbox{if } [a,b]\cap I\ne \varnothing \\
\langle \delta' +b' - a,a,a\rangle & \mbox{if } a>\max I
\end{cases}
\end{equation*}
where $I=[a'-\delta,b'-\delta]$. \footnote{Note that in the first and third cases the composition
is a constant mapping $\mathbb R\rightarrow\mathbb R$, and
can be described by many possible triplets. We chose parameters that will ensure the correctness of the backward pass
(namely, of eq.~\eqref{eq:nloglogn:backward} given later).}
\label{prop:Tlambdaij}
\end{proposition}
A proof of these facts is mechanical, and is omitted. Using induction on $j-i$,
we conclude that $T^\lambda_{ij}=\langle\delta,a,b\rangle$ for some constants $\delta,a,b$
(that may depend on $\lambda$, $i$ and $j$). 

We showed that for a fixed $\lambda$ transformation $T^\lambda_{ij}$ can be stored using $O(1)$
space and queried in $O(1)$ time. Let us now discuss how to store these
transformations for all $\lambda\in\mathbb R$; we denote the corresponding mapping $\mathbb R\times\mathbb R\rightarrow\mathbb R$ 
by $T_{ij}$. 
Let $\lambda_1,\ldots,\lambda_t$
be the breakpoint values in the non-decreasing order present in the unary
terms $g_k$ for $k\in[i+1,j]$, with $t=O(j-i)$. It follows from the previous
discussion that mapping $T_{ij}$ can be represented by a sequence 
$(\tau_0,\lambda_1,\tau_1,\ldots,\tau_{t-1},\lambda_t,\tau_t)$ where $\tau_p=(\delta_p,a_p,b_p)$.
If $\lambda_p<\lambda\le\lambda_{p+1}$ then $T^\lambda(v)=\langle\tau_p\rangle(v)$, where we assume that $\lambda_0=-\infty$
and $\lambda_{t+1}=+\infty$. 

Given sequences for $T_{ij}$ and $T_{jk}$ with $t$ and $t'$ breakpoints respectively,
we can compute the sequence for $T_{ik}$ with $t+t'$ breakpoints in $O(t+t'+1)$ time
by traversing the input sequences as in the ``merge'' operation of the {\tt MergeSort}
algorithm~\cite{Cormen}, and using Proposition~\ref{prop:Tlambdaij}(b).
Therefore, the sequence for $T_{ij}$ with $(i,j)\in\calE$
can be computed in $O(m\log m)$ time; the complexity analysis is 
the same as in the {\tt MergeSort} algorithm. The overall time for computing the sequences for all $(i,j)\in\calE$
is $O(Nm\log m)=O(n\log\log n)$.

Given such sequence, passing a message from $i$ to $j$ (i.e.\ computing $\widehat m_{j}(\lambda)=T^\lambda_{ij}(\widehat m_{i}(\lambda))$)
can be done in $O(\log m)$ time: first, we use a binary search to locate index $p$ with $\lambda_p< \lambda\le\lambda_{p+1}$,
and then return $\langle\tau_p\rangle(\widehat m_{i}(\lambda))$. 
We also need to discuss how to perform a backward pass, i.e.\ how to compute the  optimal lowest
label $x_i$ if we know message $\widehat m_{i}(\lambda)$ and the optimal lowest label $x_j\in\{0,1\}$.
Denoting $\tau_p=\langle\delta,a,b\rangle$, we can set
\begin{equation}
x_i = \begin{cases}
1 & \mbox{if } \widehat m_{i}(\lambda) < a \\
x_j & \mbox{if } \widehat m_{i}(\lambda) \in [a,b) \\
0 & \mbox{if } \widehat m_{i}(\lambda) \ge b 
\end{cases}
\label{eq:nloglogn:backward}
\end{equation}
The correctness of this rule can be verified by induction (assuming that the parameters
are computed as in Proposition~\ref{prop:Tlambdaij}); we leave it to the reader.

\paragraph{Divide-and-conquer algorithm}
We are now ready to give a formal description of the first stage.
It will be convenient to append two extra nodes $0$ and $n+1$ at the ends of the chain
with zero unary terms.
We also add edges $(0,1)$ and $(n,n+1)$ with zero weight
(and compute the sequences for mappings $T_{01}$ and $T_{n,n+1}$).
Clearly, this transformation does not change the problem.
For a subsampled node $i$ let $i^-=i-m$  be its left subsampled neighbor,
or $i^-=0$ if $i$ is the first subsampled node (i.e.\ if $i=1$). Similarly,
let $i^+$ be the right subsampled neighbor of $i$ (with $i^+=n+1$ for $i=n$). 

We will define a recursive procedure ${\tt Solve}(\calU,\calE,a,b,\ell_-,\ell_+)$.
Here $\calU$ is a non-empty set of consecutive subsampled nodes and $\calE$ is the set of edges
connecting adjacent nodes in $\calU$, containing
additionally edges $(i^-,i)$ and $(j,j^+)$ where $(i,j)=(\min \calU,\max \calU)$.
Note that $|\calE|=|\calU|+1$.
Each edge $(i,j)\in\calE$ has a pointer to
the sequence  representing mapping $T_{ij}$.
The following invariants will hold:
\begin{itemize}
\item[(a)] Minimizer $x=\argmin^-_x f(x)$ satisfies \\
(i) $x_{i^-}\le a$ (if $\ell_-=0$) or $x_{i^-}\ge b$ (if $\ell_-=1$), where $i=\min \calU$; \\
(ii) $x_k\in[a,b)$ for all $k\in \calU$; \\
(iii) $x_{j^+}\le a$ (if $\ell_+=0$) or $x_{j^+}\ge b$ (if $\ell_+=1$), where $j=\max \calU$.
\item[(b)] All breakpoints present in $T_{ij}$ for $(i,j)\in\calE$ belong to $(a,b)$.
\end{itemize}
 The output of this procedure is a minimizer $x=\argmin^-_x f(x)$ sampled at nodes $i\in \calU$,
with $x_i\in[a,b)$.
In the beginning we would call ${\tt Solve}(\calU,\calE,-\infty,+\infty,0,0)$.

Our first task is to pick the pivot value $\lambda$. 
For edge $(i,j)\in\calE$ let $\Lambda_{ij}$ be the multiset of breakpoint values
in the current sequence for $T_{ij}$. If $\Lambda_{ij}$ is empty for each $(i,j)\in\calE$
then we return solution $x_i=a$ for all $i\in  U$. 
If this is not the case then we do the following:
\begin{itemize}
\item For each $(i,j)\in \calE$ with $\Lambda_{ij}\ne\varnothing$
compute a median value $\lambda_{ij}\in\Lambda_{ij}$
(breaking the ties arbitrarily if $|\Lambda_{ij}|$ is even).
This can be done in $O(1)$ time since breakpoints in $\Lambda_{ij}$ are stored in an array in a sorted order.
\item Compute $\lambda$
 as a weighted median of the values above where $\lambda_{ij}$ comes with the weight $|\Lambda_{ij}|$.
This can be done in $O(| \calE|)$ time~\cite{Cormen}.
(This choice ensures that both of the multisets $\{ \lambda' \in \Lambda: \lambda' \ge \lambda\}$ and $\{ \lambda' \in \Lambda: \lambda' \le \lambda\}$ have at least $\frac14 |\Lambda|$ elements.)
\end{itemize}

The next step is to compute minimizer $y=\argmin^-_y g_\lambda(y)$ sampled at nodes $i\in \calU$.
As described earlier, this can be done in $O(|\calU|\log m)$ time. The first message 
is computed as follows:
if $\ell_-=0$ then $\widehat m_{i^-}(\lambda)=+\infty$,
otherwise $\widehat m_{i^-}(\lambda)=-\infty$ (where $i=\min \calU$). \footnote{This rule for $\ell_-=0$ can
be justified as follows. Constraint $x_{i^-}\le a$ means that the problem will not change if we add
unary term $f_{i^-}(x)=C|x_{i^-}-a'|$ with $C>0$ for some $a'\le a$. Such change increases the message $\widehat m_{i^-}(\lambda)$ by $C$
(since $\lambda > a'$). Since constant $C$ can be arbitrarily large, the claim follows. The case $\ell_-=1$ is analogous.}
Upon reaching node $j^+$ for $j=\max \calU$ we set its optimal label to $y_{j^+}=\ell_+$
and proceed with the backward pass.
This procedure partitions $\calU$ into sets $ \calU_1,\ldots  \calU_r$ such
that (i) $\max  \calU_s<\min  \calU_{s+1}$ for all $s$, (ii) all nodes $i\in \calU_s$ have the same label $y_i$
(which we call ``the label of $\calU_s$'' and denote as $y(\calU_s)\in\{0,1\}$), and (iii) adjacent sets $ \calU_s$ and $ \calU_{s+1}$ have different labels.

Finally, for each set $\calU_s$ we do the following. Define interval
$[a_s,b_s]=\begin{cases}[a,\lambda] & \mbox{if }y(U_s)=0 \\ [\lambda,b] & \mbox{if }y(U_s)=1\end{cases}$.
Let $\calE_s=\{(i,j)\in \calE\:|\:\{i,j\}\cap \calU_s\ne \varnothing\}$.
For each edge $(i,j)\in\calE_s$ modify the sequence for $T_{ij}$
by removing all breakpoints $\lambda'$ that do not belong to $(a_s,b_s)$ (since from now on we will 
need to pass messages from $i$ to $j$ only for values  $\lambda'\in(a,b)$). Since breakpoints are stored in a sorted order, this  takes $O(\log m)$ per $(i,j) \in \calE_s$ so $O(|\calE|\log m)$ in total.
We then make a recursive call 
${\tt Solve}(\calU_s,\calE_s,a_s,b_s,y(\calU_{s-1}),y(\calU_{s+1}))$ where 
it is assumed that $y(\calU_0)=\ell_-$ and $y(\calU_{r+1})=\ell_+$.

Note that edges $(i,j)\in\calE$ connecting adjacent sets $\calU_s$ and $\calU_{s+1}$
are split into two (one for $\calU_s$ and one for $\calU_{s+1}$). The same holds for the corresponding mappings $T_{ij}$.
Breakpoints in $T_{ij}$ that are smaller than $\lambda$ are kept
in one of the new mappings, and breakpoints that are larger than $\lambda$
are kept in the other one.

\begin{proposition}
The algorithm above has complexity $O(n\log \log n)$.
\label{prop:nloglogn}
\end{proposition}
A proof is given in Appendix~\ref{sec:proof:nloglogn}.

\paragraph{Remark 1} {\em Consider the minimization  problem~\eqref{eq:main} on a chain.
We say that it has an {\em interaction radius} $R$ if the optimal solution at node $i\in V$ depends only on unary terms $f_j$ 
and pairwise terms $f_{jk}$ for indices with $|j-i|\le R$ and $|k-i|\le R$. Note that $R\ge 1$.
It can be shown that the number of breakpoints in all messages stays bounded by some function of $R$.
This means that if $R$ is bounded by a constant, then the complexity of the presented algorithms (except for the $O(n\log\log n)$ algorithm in Sec.~\ref{sec:Convex:PiecewiseLinear})
is actually linear in $n$. In particular, complexity $O(n^2)$ for the non-convex case in Corollary~\ref{corollary:nonconvex:L1} 
becomes $O(nR)$, while the complexity $O(n\log n)$ for the convex case in Sec.~\ref{sec:Convex:PiecewiseQuadratic}
becomes $O(n\log R)$. We do not give a formal proof of these claims, so they should be treated as conjectures.

In practice we often have $R\ll n$; this happens, in particular, if  the regularization term is sufficiently weak relative to the data term.
This may explain why in the experiments given in the next section many of the algorithms empirically perform better than their worst-case complexities.
}

\section{Application examples}\label{sec:experiments}
In this section, we show how the proposed direct algorithms for total
variation minimization on trees can be used to minimize 2D total
variation based model for image processing and computer vision. In all
examples, we consider the total variation based on the $\ell_1$ norm
of the local (2D) image gradients. This allows us to rewrite the
models as the sum of one dimensional total variation problems, which
can be solved by the direct message passing algorithms proposed in
this paper. The basic idea is to perform an Lagrangian decomposition
to transform the minimization of a 2D energy to an iterative algorithm
minimizing 1D energies in each iteration. 

In~\cite{BarberoSra}, a Dykstra-like algorithm~\cite{BoyleDykstra} has
been used to iteratively minimize the TV-$\ell_2$ model. Furthermore,
the authors proposed an efficient implementation of the taut-string
algorithm, which in turn also allows also to tackle the weighted total
variation. In case all weights are equal, the method is equivalent
to Condat's algorithm. Accelerated Dykstra-like block decomposition
algorithms based on the FISTA acceleration
technique~\cite{BeckTeboulle} have been recently investigated
in~\cite{CP2015}. The authors considered different splittings of the
image domain and confirmed in numerical experiments that the
accelerated algorithms consistently outperform the unaccelerated ones.

While block-decomposition methods for minimizing the TV-$\ell_2$ have
already been proposed, using block-decomposition strategies for
solving the TV-$\ell_1$ model or (truncated) TV models subject to
nonconvex piecewise linear data terms seems to be new. Throughout this
section, we adopt the framework of structured convex-concave
saddle-point problems which can be solved by the primal-dual
algorithm~\cite{CP2011}.

\begin{itemize}
\item In case of total variation with convex quadratic unaries, we
  show that the proposed method scales linearly with the signal
  length, which is equivalent to the dynamic programming approach of
  Johnson~\cite{Johnson2013}, but better than the recently proposed
  method of Condat~\cite{Condat:13,BarberoSra} whose worst-case
  complexity is quadratic in the signal length. Furthermore, our
  algorithm can deal with weighted total variation which is not the
  case for Johnson's method. Using our proposed direct algorithms
  together with a primal-dual algorithm outperforms competing methods
  on minimizing the TV-$\ell_2$ model by one order of magnitude.
\item In case of total variation piecewise linear (or quadratic)
  unaries, we show that the empirical complexity also scales linearly
  with the signal length. We again apply the algorithm within a
  primal-dual algorithm to minimize the TV-$\ell_1$ model for 2D
  images and obtain an algorithm that outperforms the state-of-the-art by
  one order of magnitude.
\item In case of minimizing the total variation or nonconvex truncated
  total variation with nonconvex piecewise linear unaries, we show
  that the empirical worst-case bounds of our algorithms is much
  better that the theoretical worst-case bounds. We further apply the
  algorithms within a Lagrangian decomposition approach for stereo
  matching. We point out that instead of discretizing the range values
  of our problems using a discrete set of labels, we always rely on
  continuous valued solutions. Hence, the memory requirement is always
  only in the order of the 2D image size.
\end{itemize}

\subsection{TV-$\ell_2$ image restoration}

\begin{figure*}[ht!]
\subfigure[]{\includegraphics[width=0.5\textwidth]{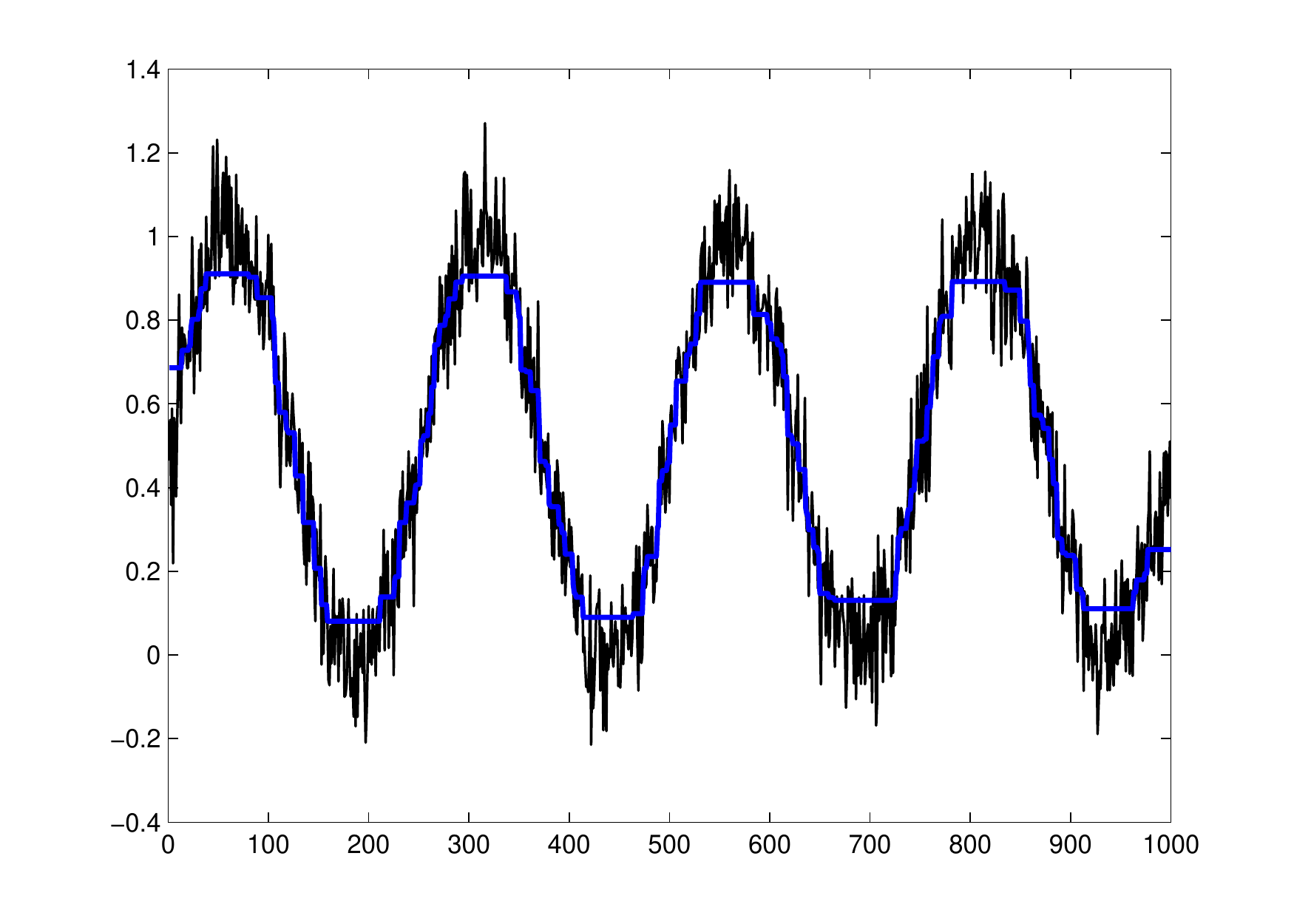}}\hfill
\subfigure[]{\includegraphics[width=0.5\textwidth]{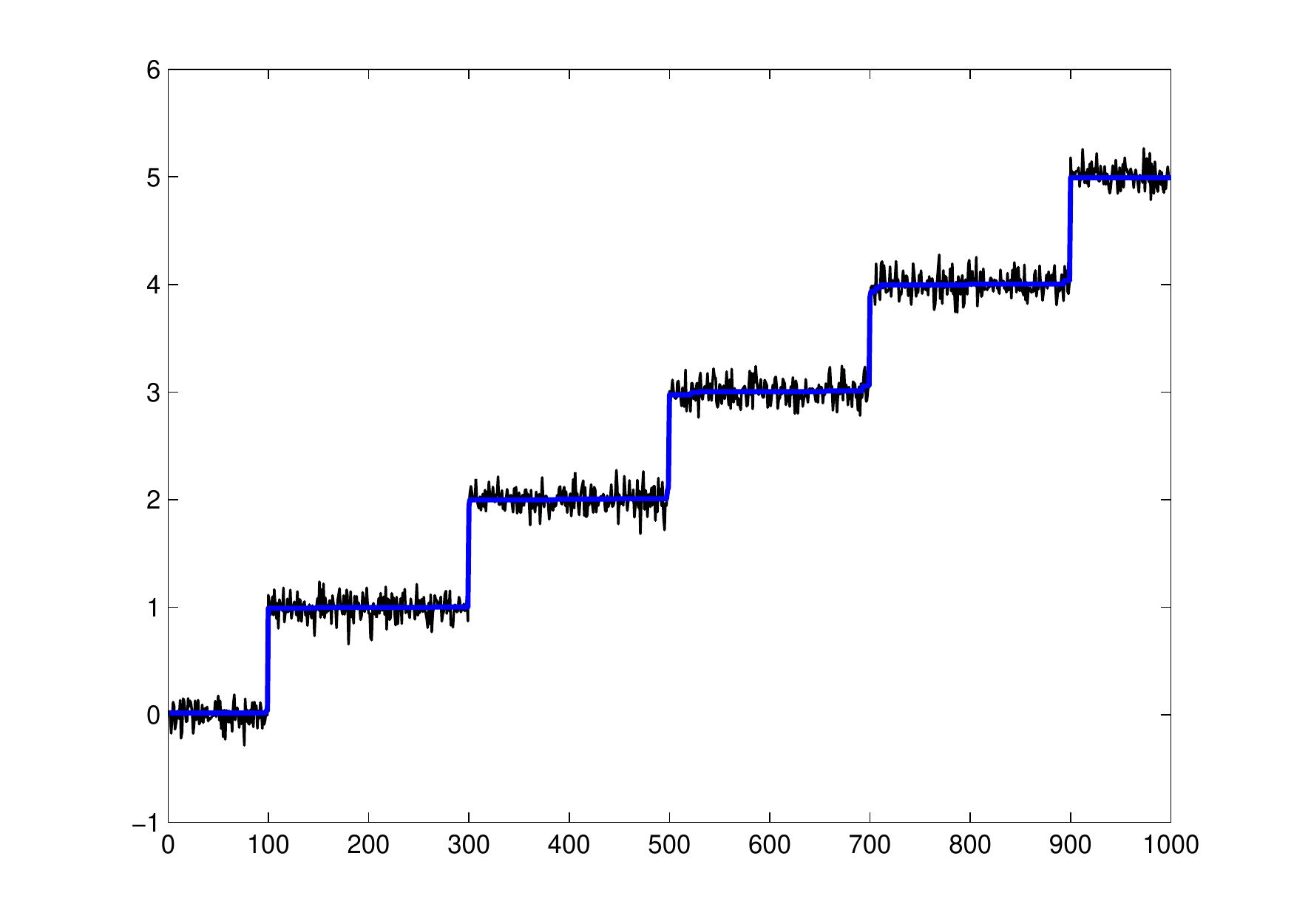}}\\
\subfigure[]{\includegraphics[width=0.5\textwidth]{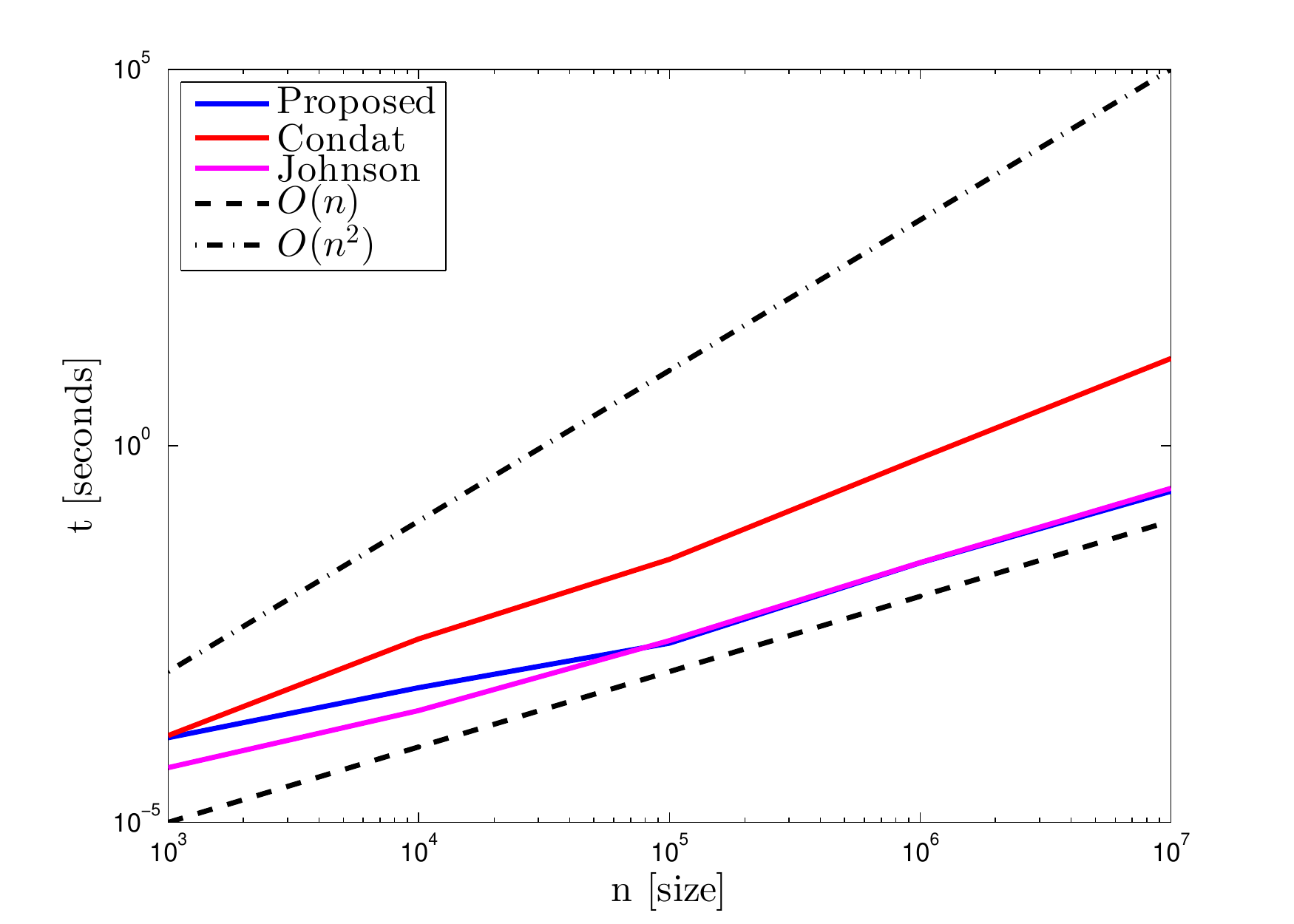}}\hfill
\subfigure[]{\includegraphics[width=0.5\textwidth]{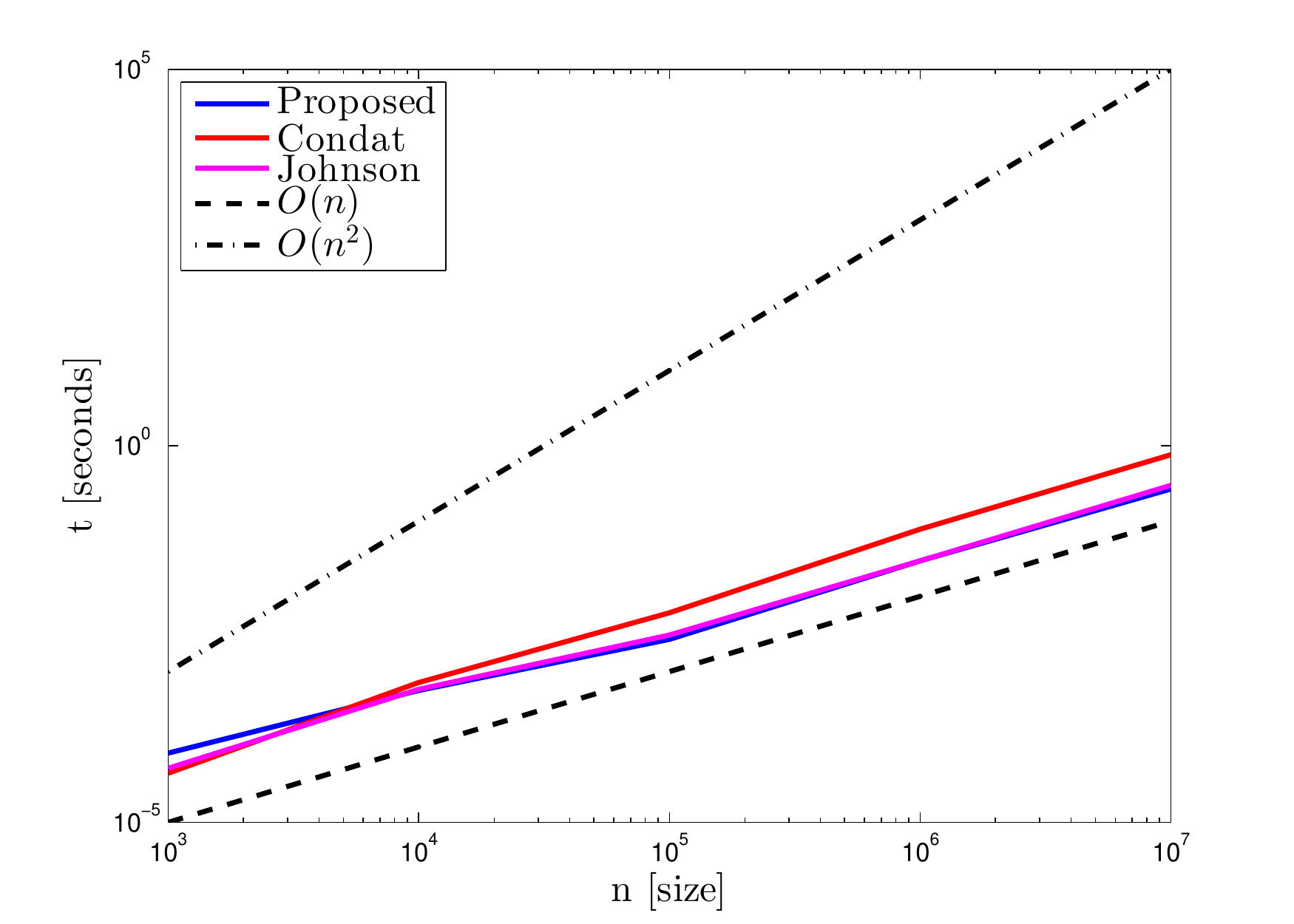}}
\caption{TV-$\ell_2$ denoising of a noisy 1D sine function and a step
  function of length $n$ containing Gaussian noise with
  $\sigma=0.1$. (a) and (b) show in black the input signal for
  $n=10^3$ and in blue the TV-$\ell_2$ regularized signal using
  $w_{i,j}=n/500$. (c) and (d) show the CPU times for different signal
  lengths $n$. Both, the proposed method and Johnson's method
  outperform Condat's method. For larger signals, the proposed method
  appears slightly more efficient than Johnson's method. In case of
  the step function also Condat's method seems to be
  competitive.}\label{fig:sine-tv-l2}
\end{figure*}

First, we provide a comparison of the proposed message passing
algorithm for solving total variation with convex quadratic unaries
(TV-$\ell_2$) on chains to the competing methods of
Condat~\cite{Condat:13} and Johnson~\cite{Johnson2013}. The problem is
written as
\begin{equation}\label{eq:1D-ROF}
\min_{x\in \mathbb R^{n}} \sum_{(i,j) \in E } w_{ij}|x_{i}-x_{j}| +
\frac12 \sum_{i \in V} (x_{i}-f_{i})^2,
\end{equation}
where $n$ is the length of the signal and $w_{ij}$ are the pairwise
weights. Fig.~\ref{fig:sine-tv-l2} provides a comparison of the
proposed algorithm to the aforementioned competing methods based on
regularizing a smooth sine-like function and a piecewise constant step
function. All three methods have been implemented in C++ and executed
on a single CPU core. The implementations of Condat and Johnson have
been provided by the authors. For the smooth sine function, the
experiments show that while Condat's method has an almost quadratic
worst case time complexity, the method of Johnson and the proposed
method have a linear time complexity. On the piecewise constant step
function, Condat's method appears to perform better but still slower
than Johnson's method and the proposed method.

\begin{figure*}
\subfigure[]{\includegraphics[width=0.45\textwidth]{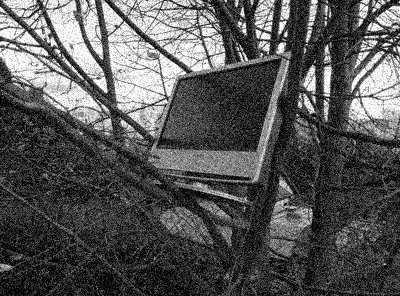}}\hfill
\subfigure[]{\includegraphics[width=0.45\textwidth]{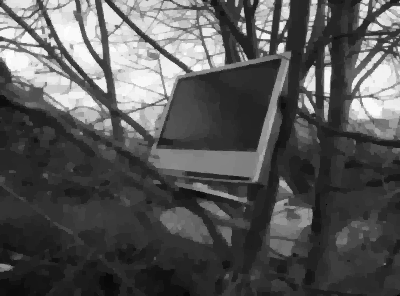}}\\
\subfigure[]{\includegraphics[width=0.5\textwidth]{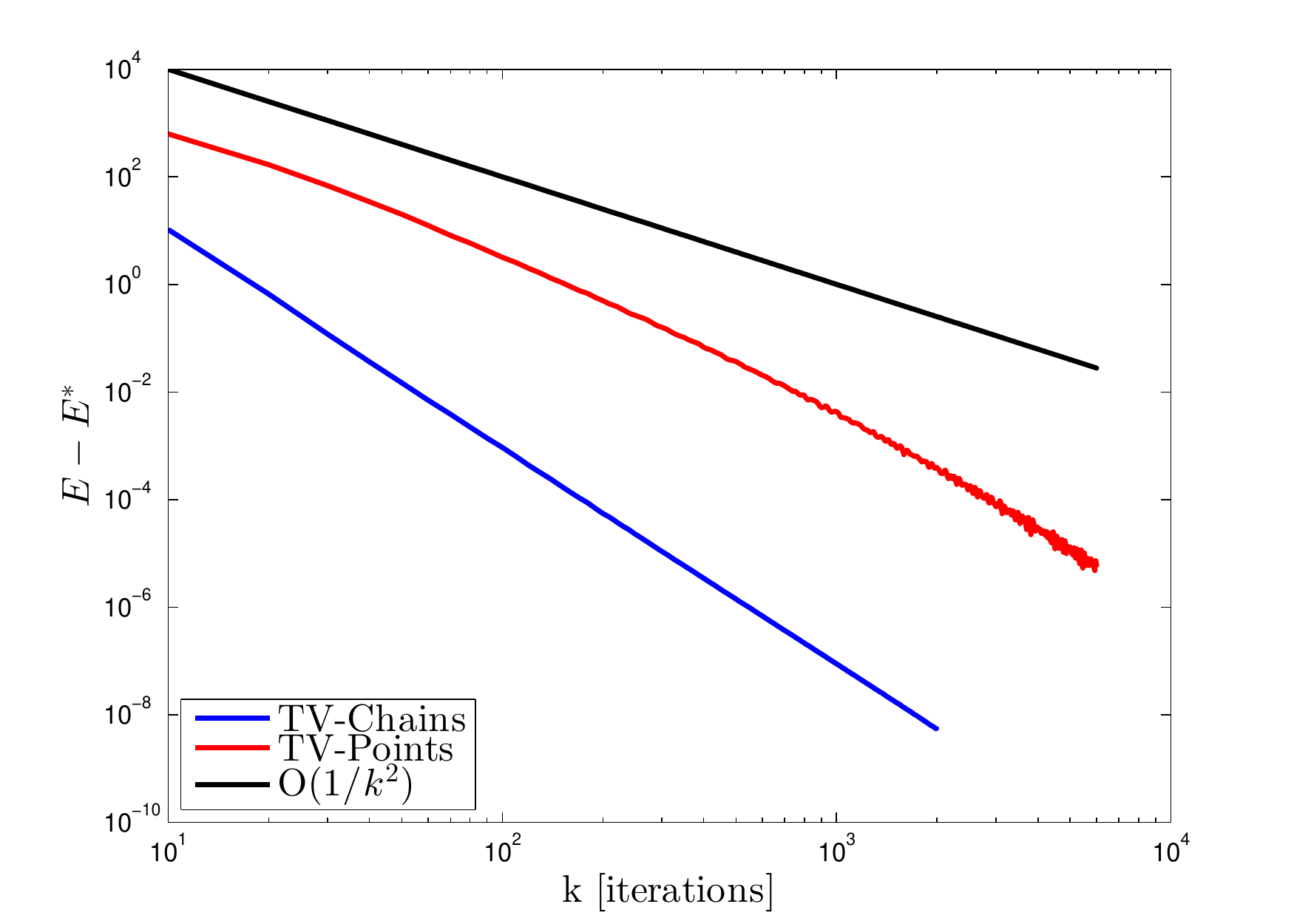}}\hfill
\subfigure[]{\includegraphics[width=0.5\textwidth]{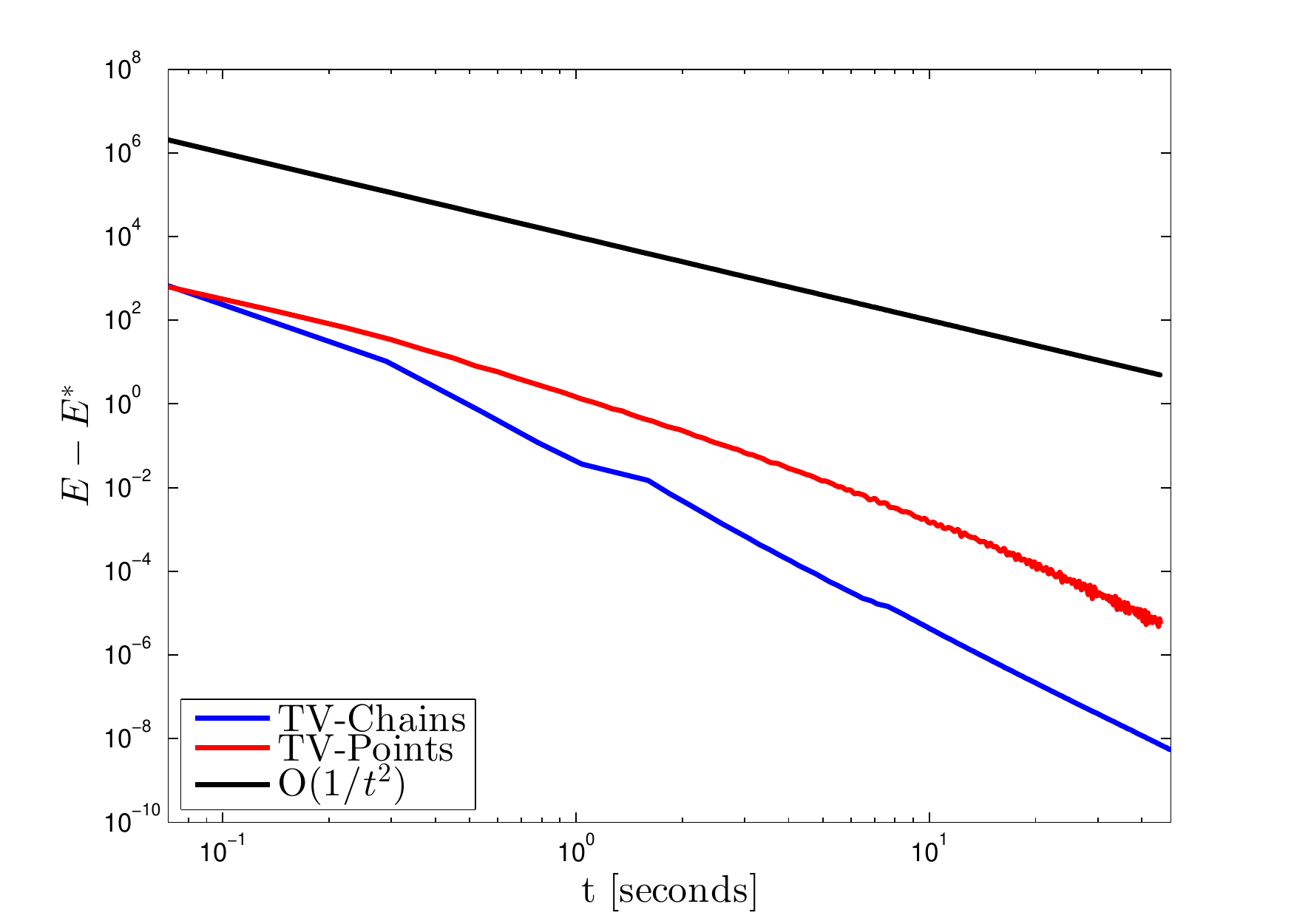}}
\caption{TV-$\ell_2$ denoising: (a) shows the ``TV-Tree'' test image
  of size $400\times 296$, which has been degraded by adding zero-mean
  Gaussian noise with standard deviation $\sigma = 25/255$. (b) shows
  the result of TV-$\ell_2$ denoising using $w_{ij} = 0.1$ for all
  $i,j$. (c) shows a comparison in terms of iterations and (d) shows a
  comparison in terms of CPU time.}\label{fig:monarch-tv-l2}
\end{figure*}

In the second example, we consider the Rudin-Osher-Fatemi
(ROF)~\cite{ROF} model for total variation image restoration of 2D
images. We point out that the ROF problem is a very fundamental
problem since besides image denoising it can also be used to compute
graph cuts~\cite{Chambolle:EMMCVPR05}.

We consider a given image $f \in \mathbb R^{mn}$ which is defined on a
regular 2D graph of $m\times n$ vertices (pixels). The model is
written as the convex minimization problem
\[
\min_{x\in \mathbb R^{mn}} TV_h(x) + TV_v(x) + \frac12 \sum_{i \in V}
(x_{i}-f_{i})^2,
\]
where $V$ is the set of nodes (pixels), with $|V|=mn$. $TV_h(x)$ and
$TV_v(x)$ refer to the total variation in horizontal and vertical
direction, which are given by
\[
TV_{h,v}(x) = \sum_{(i,j) \in E_{h,v}} w_{ij}|x_{i}-x_{j}|, 
\]
where $E_h$ and $E_v$ correspond to the sets of vertical and
horizontal edges defined on the 2D graph.  Now, we perform a
Lagrangian decomposition and rewrite the above problem as
\[
\Lambda(x,x',y) = TV_h(x) + TV_v(x') + \frac12 \norm[2]{x-f}^2 +
\scal{x-x'}{y},
\]
where $y \in \mathbb R^{mn}$ is a Lagrange multiplier and
$\scal{\cdot}{\cdot}$ denotes the usual scalar product. We proceed by
observing that the convex conjugate $TV_v^*$ of the function $TV_v$ is
given by
\[
TV_v^*(y) = \sup_{x' \in \mathbb R^{mn}} \scal{y}{x'} - TV_v(x').
\]
Substituting back in the Lagrangian yields the following saddle-point
problem:
\begin{equation}\label{eq:rof-saddle}
\min_x \max_y \scal{x}{y} + TV_h(x) + \frac12\norm[2]{x-f}^2 -
TV_v^*(y).
\end{equation}
This problem can be solved by the first-order primal-dual algorithm
proposed in~\cite{CP2011}, which in our setting is given by
\begin{equation}\label{eq:pd}
\begin{cases}
  y^{k+1} = \mathrm{prox}_{\sigma_k TV_v^*} \left(y^k + \sigma_k (x^k + \theta_k(x^k-x^{k-1}))\right)\\
  x^{k+1} = \mathrm{prox}_{\tau_k (TV_h + \frac12 \norm{\cdot-f}^2)}
  \left(x^k - \tau_k (y^{k+1})\right),
\end{cases}
\end{equation}
where $\tau_k,\sigma_k,\theta_k$ are positive step size parameters
such that $\tau_k\sigma_k=1$, $\theta_k \in (0,1]$. Since the
saddle-point problem is $1$-strongly convex in the primal variable
$x$, we can apply the accelerated variant of the primal-dual
algorithm, ensuring an optimal $O(1/k^2)$ convergence of the
rimal-dual gap (see~\cite{CP2014}). Observe that since the linear
operator in the bilinear term in~\eqref{eq:rof-saddle} is the
identity, the primal-dual algorithm is equivalent to an (accelerated)
Douglas-Rachford splitting (see~\cite{CP2011}).

In order to make the primal-dual algorithm implementable, we need to
efficiently compute the proximal maps with respect to the functions
$TV_h + \frac12\norm{\cdot-f}^2$ and $TV_v^*$. It can be checked that
the proximal map for the primal function is given by
\[
  \mathrm{prox}_{\tau (TV_h + \frac12\norm{\cdot-f}^2)}(\xi) =
  \arg\min_{x} TV_h(x) + \frac{1+\tau^{-1}}{2} \norm[2]{x-
    (1+\tau^{-1})^{-1}(f + \tau^{-1}\xi)}^2,
\]
for some $\xi \in \mathbb R^{mn}$ and $\tau > 0$. Its solution can be
computed by solving $m$ independent problems of the
form~\eqref{eq:1D-ROF}. In order to compute the proximal map with
respect to the dual function, we make use of the celebrated Moreau
identity
\begin{equation}\label{eq:prox-tv-conj}
  y = \mathrm{prox}_{\sigma TV_v^*}(y) +
  \sigma \cdot \mathrm{prox}_{\sigma ^{-1} TV_v} \left(\sigma^{-1} y\right)\,,
\end{equation}
which shows that the proximal map with respect to $TV_v^*$ can be
computed by computing the proximal map with respect to $TV_v$.
\begin{equation*}
  \mathrm{prox}_{\sigma TV_v^*}(\eta) = \eta - \sigma \cdot
  \arg\min_{x} TV_v(x) + \frac{\sigma}{2} \norm[2]{x - \sigma^{-1} \eta}^2,
\end{equation*}
for some given $\eta \in \mathbb R^{mn}$ and $\sigma > 0$. The
proximal map again reduces to $n$ independent problems of the
form~\eqref{eq:1D-ROF}. According to Sec.~\ref{sec:Convex:Quadratic},
the total complexity for computing the proximal maps is $O(mn)$ and
hence linear in the number of image pixels. Furthermore, the
computation of the independent 1D subproblems can be done fully in
parallel.

Fig.~\ref{fig:monarch-tv-l2} presents the results of a performance
comparison between the proposed accelerated primal-dual algorithm by
solving TV-$\ell_2$ problems on chains (TV-Chains) and the
state-of-the art primal-dual algorithm proposed in~\cite{CP2011} which
is based on a pointwise decomposition (TV-Points). Both algorithms
were implemented in Matlab, while for TV-Chains, the solution of the
1D subproblems was implemented in C++. The figure shows that TV-Chains
converges significantly faster than TV-Points both in terms of
iterations and CPU time and hence significantly improves the
state-of-the art (approx. one order of magnitude).

\subsection{TV-$\ell_1$ image restoration}

Next, we consider again total variation minimization but now with a
$\ell_1$ data fitting term. The minimization problem is given by
\[
\min_{x\in \mathbb R^{mn}} TV_h(x) + TV_v(x) + \sum_{i \in V}
|x_i-f_i|.
\]
It is well-known that the TV-$\ell_1$ model performs significantly
better compared to the TV-$\ell_2$ model in presence of non-Gaussian
noise. However, being a completely nonsmooth optimization problem it
is also significantly more challenging to minimize.

\begin{figure*}[ht!]
\subfigure[]{\includegraphics[width=0.5\textwidth]{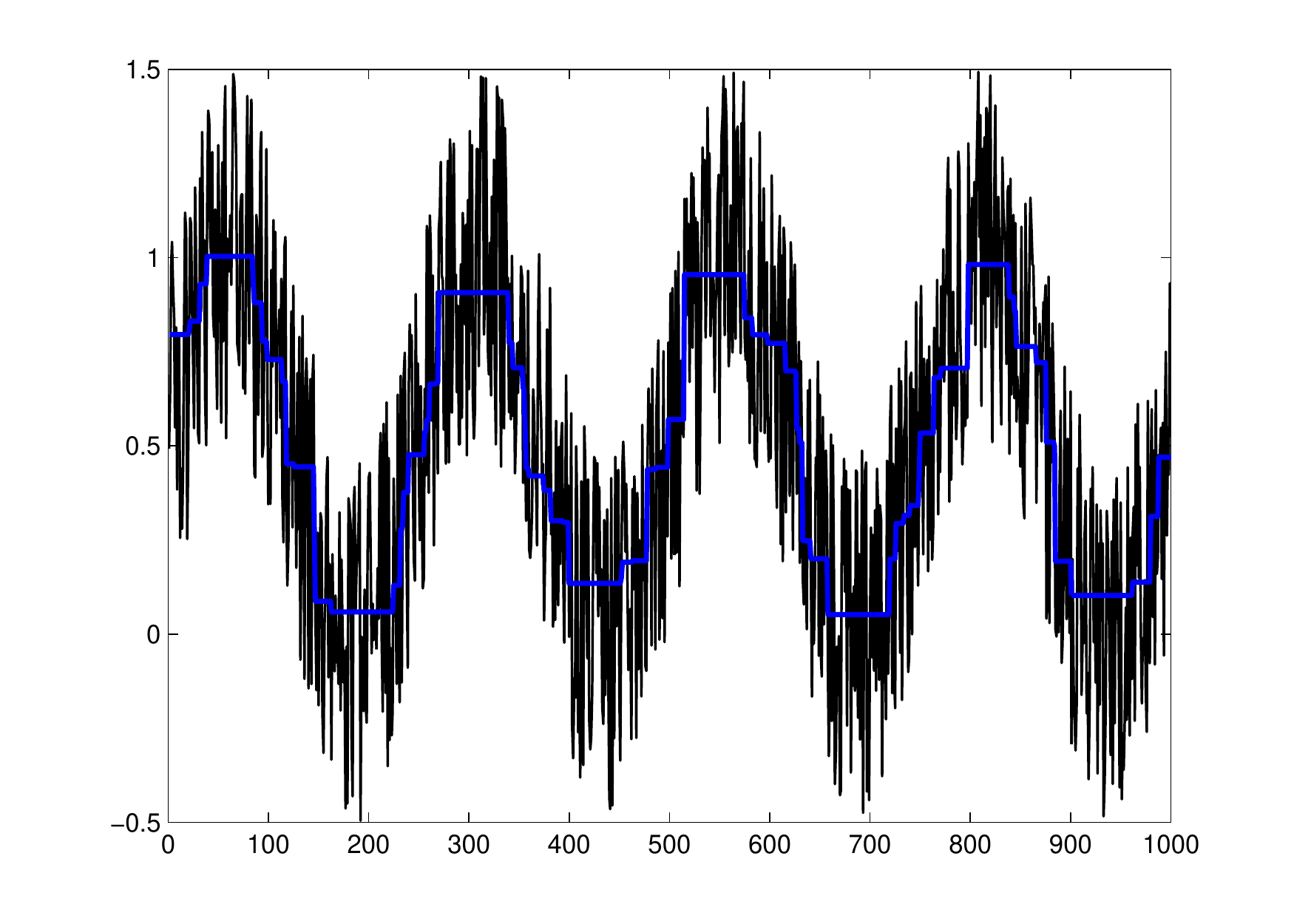}}\hfill
\subfigure[]{\includegraphics[width=0.5\textwidth]{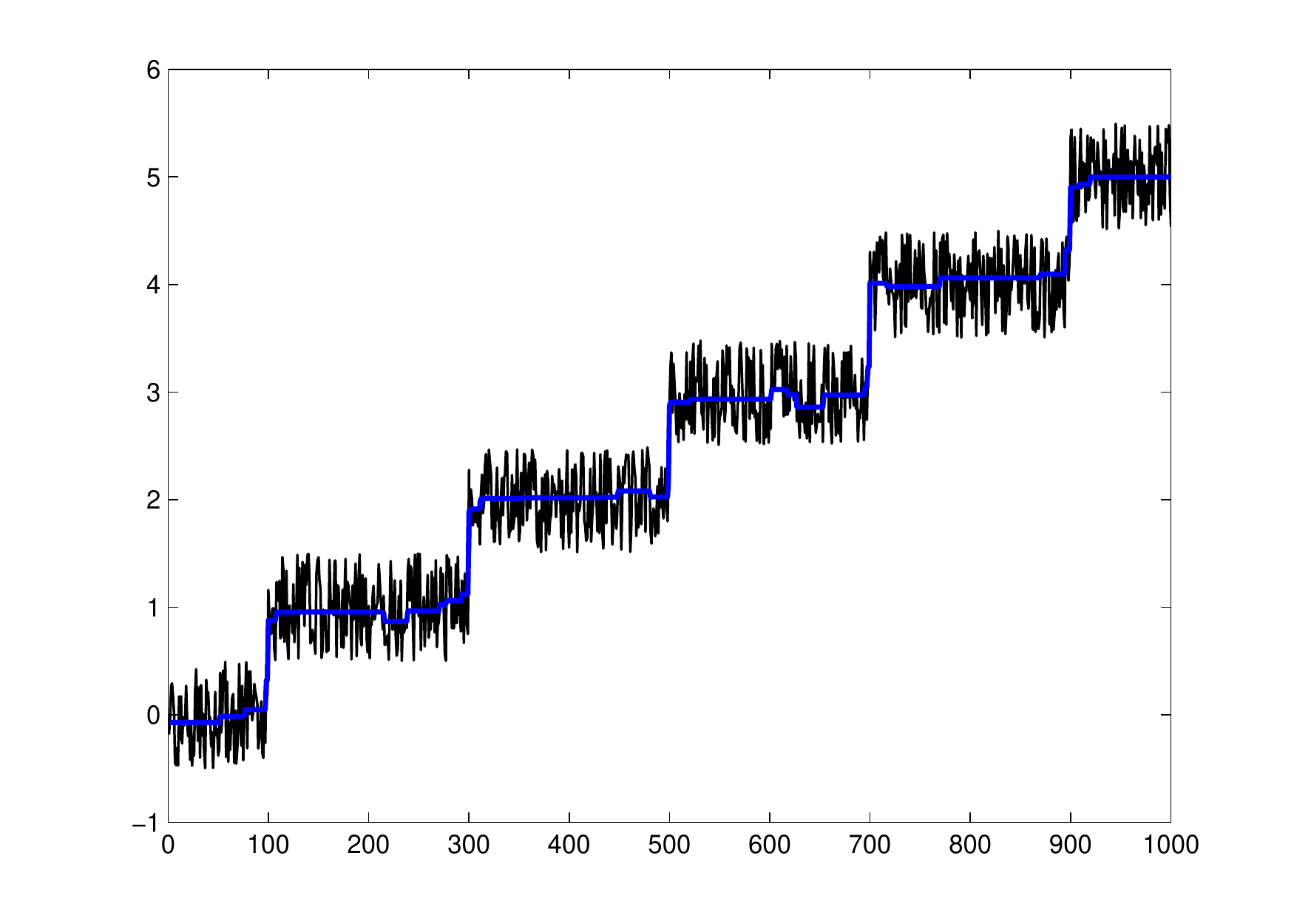}}\\
\subfigure[]{\includegraphics[width=0.5\textwidth]{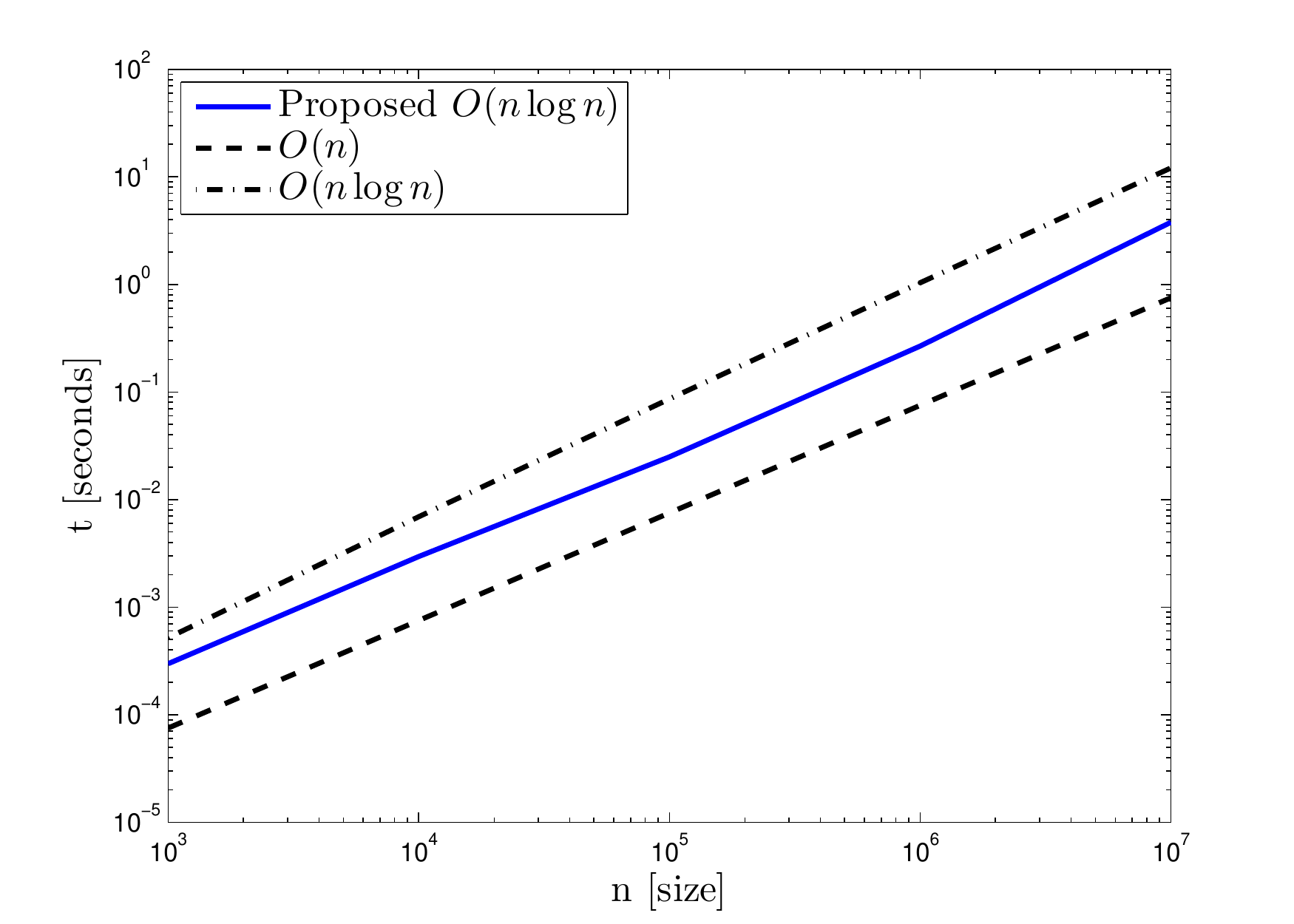}}\hfill
\subfigure[]{\includegraphics[width=0.5\textwidth]{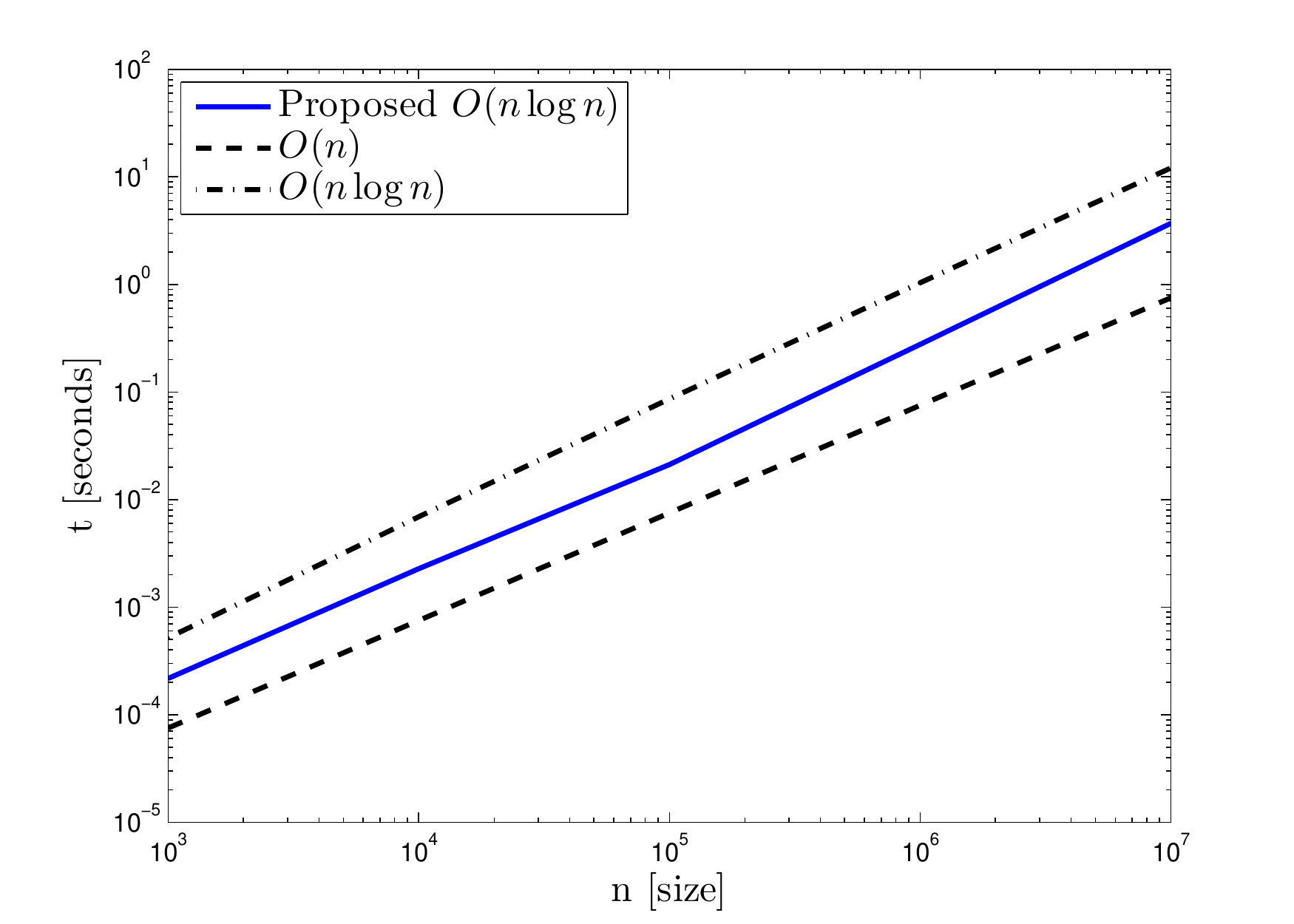}}
\caption{TV-$\ell_1$ denoising of a 1D sine function and a step
  function of length $n$ containing zero-mean uniformly distributed
  noise with magnitude $1/2$. (a) and (b) show in black the input signal for
  $n=10^3$ and in blue the TV-$\ell_1$ regularized signal using
  $w_{i,j}=200/n$. (c) and (d) show the CPU times for different signal lengths
  $n$. One can see that the empirical complexity of the proposed
  direct algorithm for computing the proximal map with respect to the
  1D TV-$\ell_1$ model is between $O(n)$ and $O(n\log n)$.}\label{fig:sinel1}
\end{figure*}

\begin{figure*}
\subfigure[]{\includegraphics[width=0.45\textwidth]{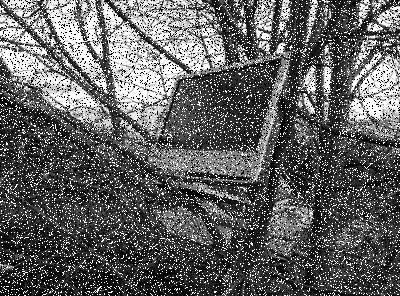}}\hfill
\subfigure[]{\includegraphics[width=0.45\textwidth]{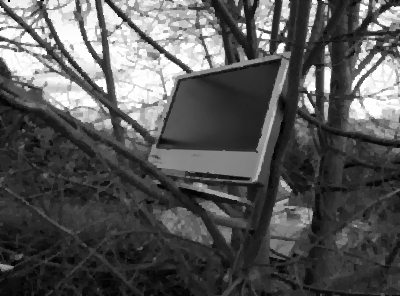}}\\
\subfigure[]{\includegraphics[width=0.5\textwidth]{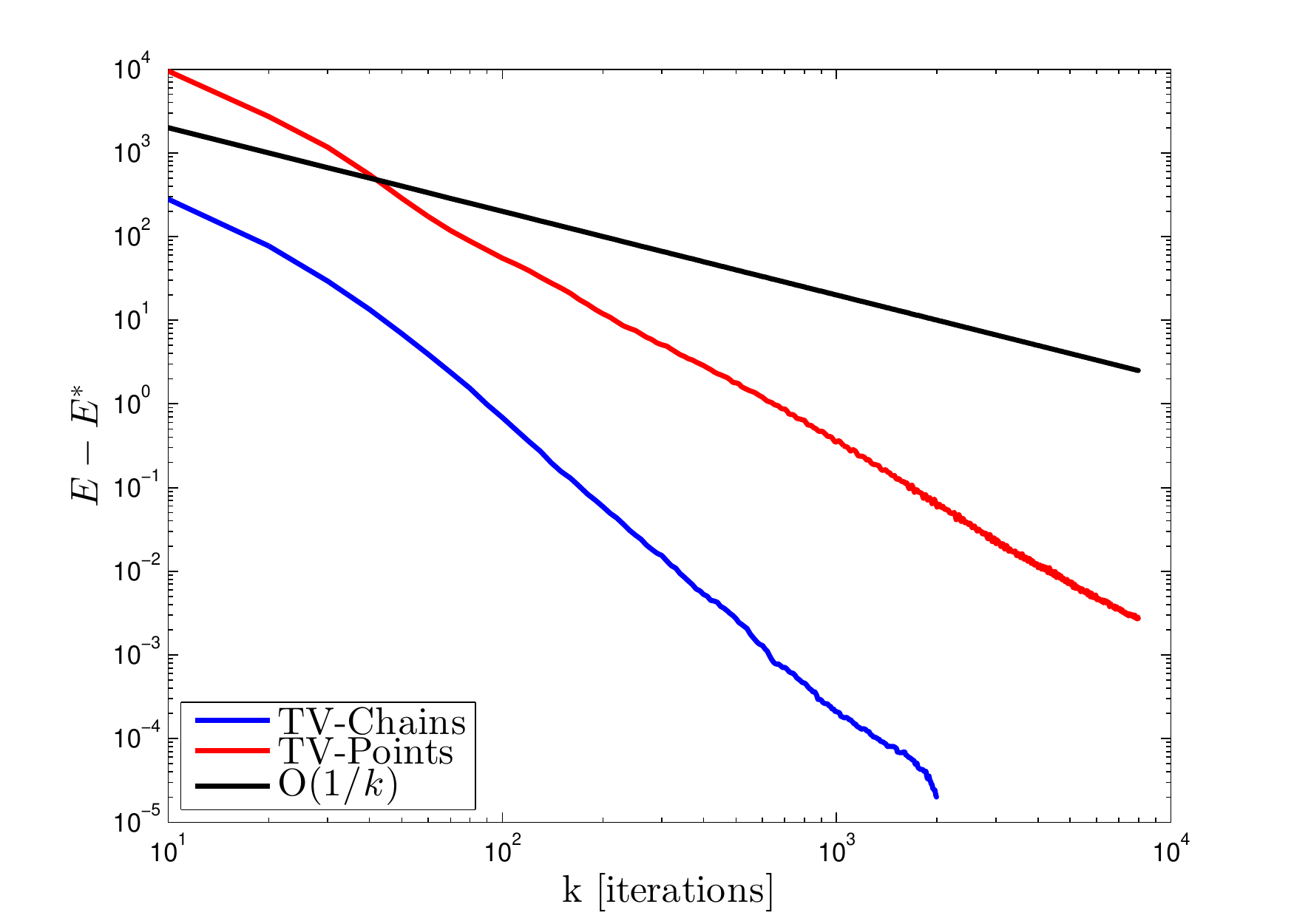}}\hfill
\subfigure[]{\includegraphics[width=0.5\textwidth]{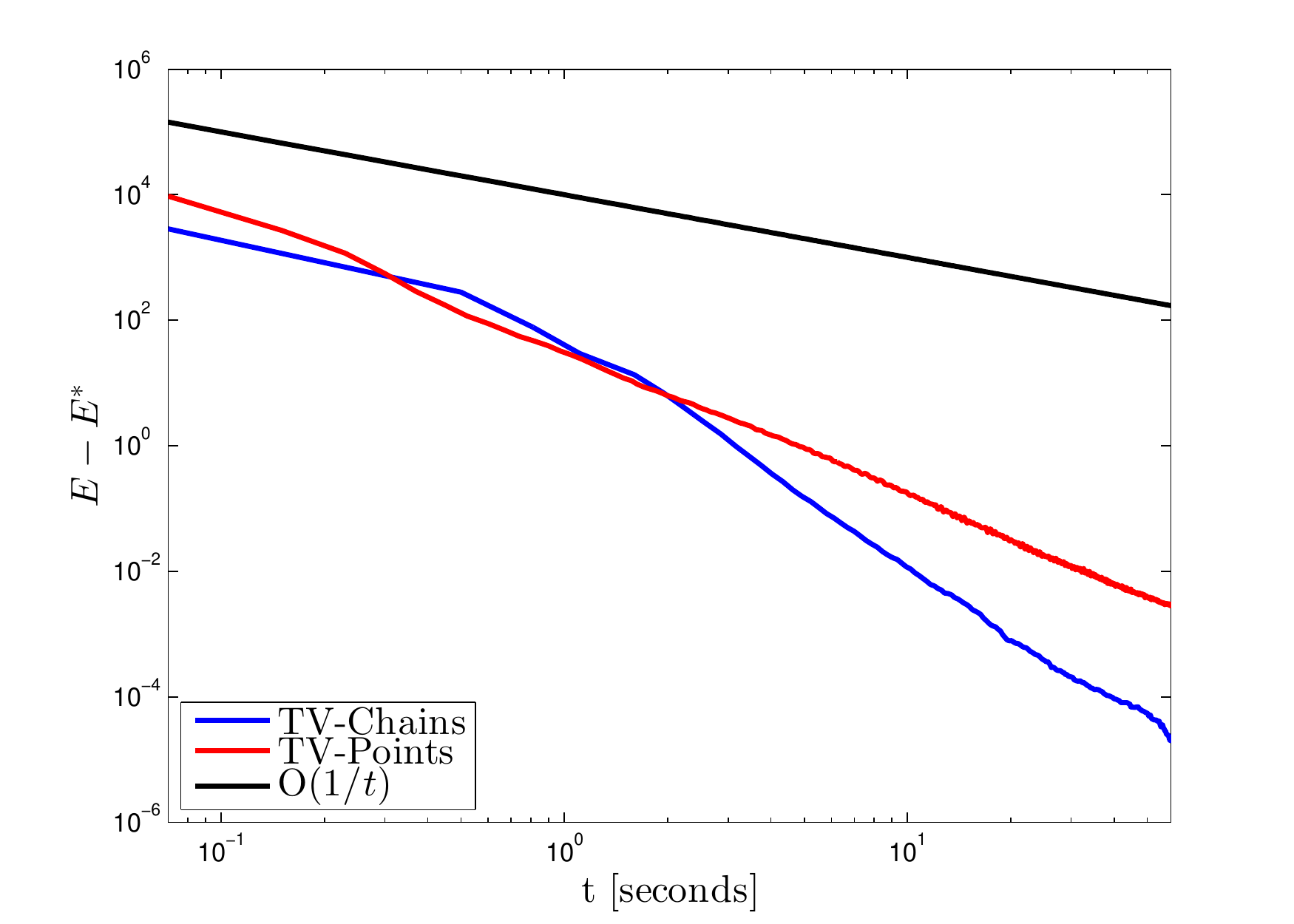}}
\caption{TV-$\ell_1$ denoising: (a) shows the ``TV-tree'' test image
  of size $400\times 296$, which has been degraded by 25\%
  salt\&pepper noise. (b) shows the result of TV-$\ell_1$ denoising
  using $w_{ij} = 0.55$ for all $i,j$. (c) shows the convergence rate
  in terms of iterations and (d) shows the convergence rate in terms
  of CPU time.}\label{fig:monarch-tv-l1}
\end{figure*}

In order to apply the direct 1D algorithms proposed in this paper to
minimize the TV-$\ell_1$ model we consider a splitting in the same
spirit as in the previous section.
\[
\min_x \max_y \scal{x}{y} + TV_h(x) + \frac12\norm[1]{x-f} -
TV_v^*(y).
\]
We solve the saddle-point problem again by using the primal-dual
algorithm~\eqref{eq:pd}. To make the algorithm implementable, we need
fast algorithms to solve the proximity operators with respect to both
the primal and dual functions. The proximity operator with respect to
the primal function is given by
\begin{equation}\label{eq:prox-pq}
  \mathrm{prox}_{\tau (TV_h + \norm[1]{\cdot-f})}(\xi) =
  \arg\min_{x} TV_h(x) + \norm[1]{x-f} + \frac{1}{2\tau}
  \norm[2]{x-\xi}^2,
\end{equation}
for some point $\xi \in \mathbb R^{mn}$ and $\tau > 0$.  Computing
this proximity operator reduces to minimizing $m$ independent total
variation problems subject to piecewise quadratic unaries. According
to Sec.~\ref{sec:Convex:PiecewiseQuadratic}, one subproblem can be
computed in $O(n\log n)$ time. The proximity operator with respect to
$TV_v^*$ is equivalent to the proximity operator
in~\eqref{eq:prox-tv-conj} and hence it reduces to $n$ independent 1D
TV problems subject to quadratic unaries. Hence, the overall
complexity for one iteration of the primal dual algorithm is
$O(mn\log n)$.

We first evaluate the empirical complexity of the direct algorithm for
minimizing the total variation with piecewise quadratic unaries which
is used in~\eqref{eq:prox-pq} to compute the proximal map with respect
to the 1D TV-$\ell_1$ problems. For this we again consider a
discretized sine function and a step function with different signal
lengths $n$ and we added zero-mean uniformly distributed noise with
magnitude $1/2$ (see Fig.\ref{fig:sinel1}. The pairwise weights were
set to $w_{i,j}=200/n$. We also set the quadratic part of the function
close to zero ($10^{-6}$) in order to be able to successfully restore
the signal. From Fig.~\ref{fig:sinel1}, one can see that the empirical
performance of the proposed algorithm for piecewise quadratic unaries
is between $O(n)$ and its worst case complexity of $O(n\log n)$. We
also compared with the $O(n \log \log n)$ direct algorithm for convex
piecewise linear unaries and it turned out that the practical
performance is about the same.

Fig.~\ref{fig:monarch-tv-l1} shows a comparison of the proposed
primal-dual algorithm based on chains (TV-Chains) to the primal-dual
algorithm based on a points-based splitting (TV-Points)~\cite{CP2011}.
Although theoretically not justified, we again used varying step sizes
in case of TV-Chains to accelerate the convergence. For TV-Points, the
acceleration scheme did not work. Both algorithms were again
implemented in Matlab, while for TV-Chains, the solution of the
proximal operators were implemented in C++.  The comparison shows that
TV-Chains needs far less iterations compared to TV-Points and it is
also significantly more efficient in terms of the CPU time
(approx. one order of magnitude).

\subsection{TV-nonconvex}

Finally, we consider total variation minimization subject to nonconvex
piecewise linear unaries. Such problems arise for example in stereo
and optical flow estimation. The general form of the minimization
problem we consider here is given by
\begin{equation}\label{eq:tv-lin}
  \min_{x\in \mathbb R^{mn}} \calP(x) = TV^C_h(x) + TV^C_v(x) + \sum_{i \in V} f_i(x_i),
\end{equation}
where $f_i$ are continuous piecewise linear functions, which are
defined by a set of $t+1$ slopes $(s_l)_{l=0}^t$, and a corresponding
set of $t$ break-points $(\lambda_kl)_{l=1}^t$. $TV^C_{h,v}(x)$ refers
to truncated total variation defined by
\begin{equation}\label{eq:ttv}
TV^C_{h,v}(x) = \sum_{(i,j) \in E_{h,v}} w_{ij}\cdot\min(C,
|x_{i}-x_{j}|),
\end{equation}
where $w_{ij}$ are edge weights and $C$ is some positive
constant. Observe that convex total variation is obtained for
$C=\infty$. We again perform a splitting into horizontal and vertical
1D problems and consider the Lagrangian
\[
\min_{x_{h,v}}\max_y\Psi_h(x_h) + \Psi_v(x_v) + \scal{x_h-x_v}{y}.
\]
where
\[
\Psi_{h,v}(x) = TV^C_{h,v}(x) + \frac12 \sum_{i \in V} f_i(x_i).
\]
The reason for splitting the nonconvex term into two parts is that
there is a higher chance that part of the nonconvexity are absorbed by
the convexity of the regularization terms. Observe that while the
problem is nonconvex in $x_h$ and $x_v$ it is concave in $y$ since it
is a pointwise maximum over linear functions.

In contrast to the application of the convex conjugate utilized in the
two previous examples, we consider here a direct application of the
primal-dual algorithm~\cite{CP2011} to the Lagrangian function. The
algorithm takes the following form:
\[
\begin{cases}  
  x_{h}^{k+1} = \mathrm{prox}_{\tau_k \Psi_{h}} \left(x_h^k - \tau_k \bar y^k \right)\\
  x_{v}^{k+1} = \mathrm{prox}_{\tau_k \Psi_{v}} \left(x_v^k + \tau_k \bar y^k \right)\\
  y^{k+1} = y^k + \sigma_k (x_h^{k+1} - x_v^{k+1})\\
  \bar y^{k+1} = y^{k+1} + \theta^k(y^{k+1}-y^{k}).
\end{cases}
\]
The proximal maps with respect to the nonconvex functions
$\Psi_{h,v}$ are computed by adding a piecewise linear approximation
of the quadratic proximity term
$\frac{1}{2\tau_k}\norm[2]{\cdot-x_{h,v}^k}^2$ to the functions
$\Psi_{h,v}$ and solving the resulting independent 1D problems using
the direct algorithm for minimizing the (truncated) total variation
subject to (nonconvex) piecewise linear unaries which has been
presented in Sec.~\ref{sec:nonconvex}.

Due to the nonconvexity in the primal objective, the primal-dual
algorithm is not guaranteed to converge. However, we observe
convergence when gradually decreasing the step size parameter $\tau_k$
during the iterations. The intuition behind this strategy is that by
gradually decreasing the primal step size, the primal-dual algorithm
approaches a (regularized) dual algorithm, applied to the (concave)
dual objective. We found that the rule $\tau_k = \tau_0/k$, $\tau_0
\approx 100...1000$ works well in practice. The dual step size is set
to $\sigma_k = 1/(\tau_k L^2)$, where the Lipschitz constant $L$ is
computed as $L=\sqrt{2}$. The relaxation parameter $\theta_k$ is
constantly set to $\theta_k=1$.

We applied problem~\eqref{eq:tv-lin} to disparity estimation in stereo
images. The stereo image pair is the ``Motorcycle'' data set of size
$1000\times 1482$ pixels, which is taken from the recently introduced
Middlebury stereo data set~\cite{Scharstein2014} (see
Fig.~\ref{fig:moto}). The stereo data term (piecewise linear functions
$f_i$ in~\eqref{eq:tv-lin}) and the edge weights ($w_{ij}$
in~\eqref{eq:ttv})are set identically to the stereo experiment
described in~\cite{CP2015}. The piecewise linear matching function is
computed using $126$ break points, which corresponds to a disparity
range of $[0,125]$.

\begin{figure*}[ht!]
  \subfigure[]{\includegraphics[width=0.45\textwidth]{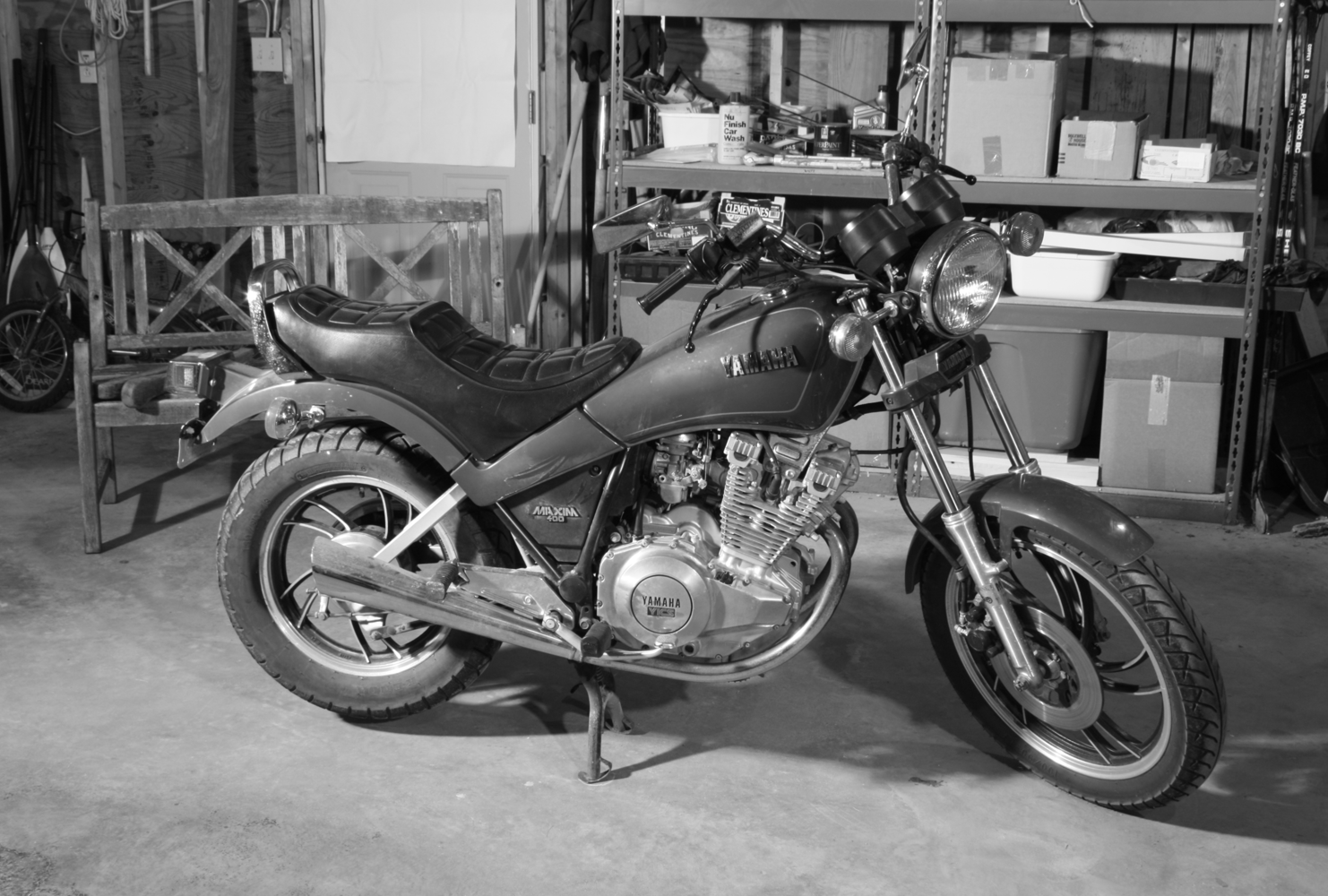}}\hfill
  \subfigure[]{\includegraphics[width=0.45\textwidth]{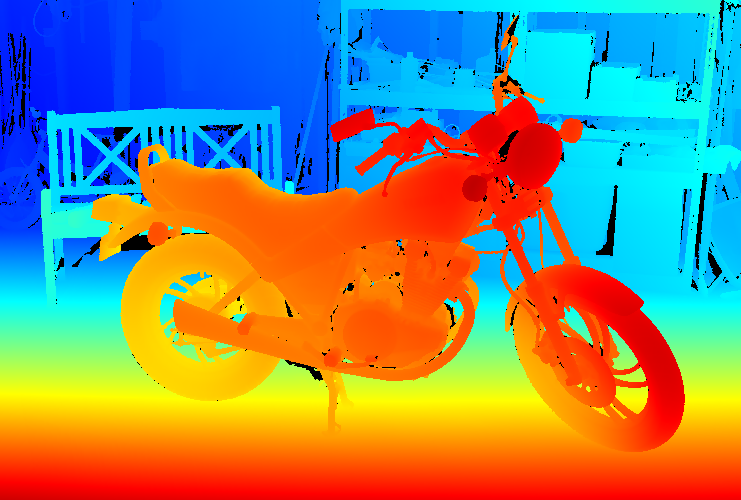}}
  \caption{``Motorcycle'' stereo data set used in the experiment. (a)
    Shows the left input image of size $1000\times 1482$ and (b) shows
    the color coded ground truth disparity map.}\label{fig:moto}
\end{figure*}

In the first experiment, we evaluate the practical performance of our
proposed dynamic programming algorithms for minimizing the convex and
nonconvex total variation subject to nonconvex piecewise linear
unaries. For this, we consider different sizes of the stereo image
pair and recorded the average time of computing the solutions of the
horizontal lines during the first iteration of the algorithm. The
number of break points is kept constant in all problems. The
worst-case complexity for solving one problem of size $n$ is $O(n^2)$
in case of convex total variation and exponential in case of nonconvex
truncated total variation. The resulting timings are presented in
Fig.~\ref{fig:timing-nonconvex}. One can clearly see that the
practical performance of the algorithm is significantly better than
the theoretical worst-case complexities (see
Sec.~\ref{sec:nonconvex}).

\begin{figure}[ht!]
  \centering
  \includegraphics[width=0.5\textwidth]{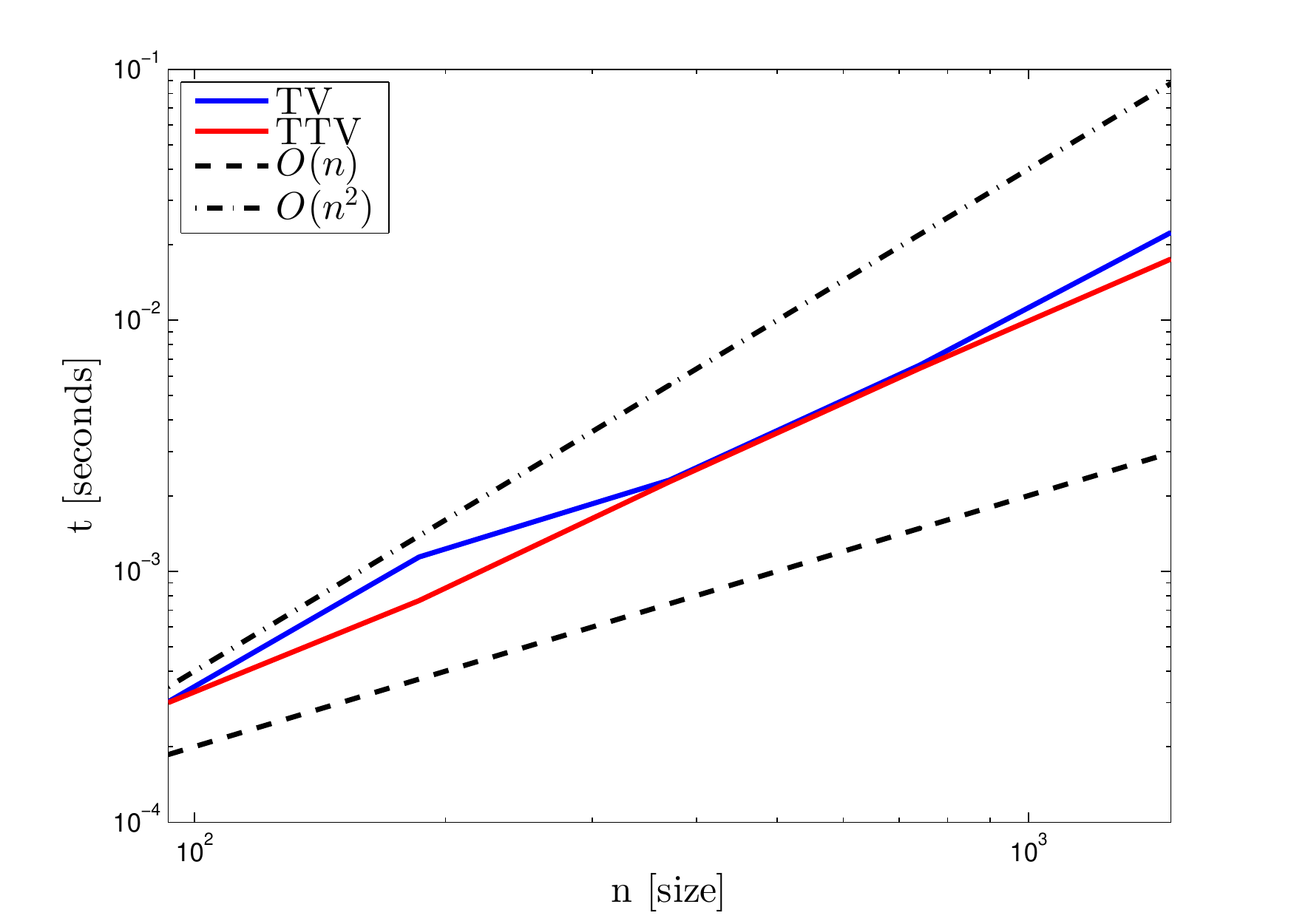}
  \caption{Evaluation of the practical performance of the proposed
    dynamic programming algorithm for minimizing 1D total variation
    subject to nonconvex piecewise linear unaries and using different
    sizes $n$. The practical performance for both convex total
    variation (TV) and nonconvex truncated total variation (TTV) is
    significantly better compared to the theoretical worst-case
    complexities presented in
    Sec.~\ref{sec:nonconvex}.}\label{fig:timing-nonconvex}
\end{figure}

In our second experiment, we conduct exactly the same stereo
experiment as in~\cite{CP2015}. However, instead of computing the
globally optimal solution by means of a lifting approach in 3D, we
directly solve the nonconvex 2D Lagrangian problem. We use either
convex total variation (TV) or truncated total variation (TTV) where
we set the truncation value to be $C=10$. The primal variables
$x_{h,v}$ are initialized by the solutions of the 1D problems
(assuming no coupling between the horizontal and vertical chains).

Fig.~\ref{fig:stereo-disp} shows a comparison between our proposed
Lagrangian decomposition and the globally optimal solution obtained
from~\cite{CP2015}. Observe that the primal energy of the Lagrangian
decomposition method quickly decreases during the first
iterations. Suprisingly, we can approach the lower bound up to a very
small error after a larger number of iterations. We also plot the
color coded disparity maps corresponding to the average solution $\bar
x^k = (x_h^k+x_v^k)/2$. While the solution after the first iteration
still shows some streaking artifacts, the solution obtained after only
$10$ iterations is visually almost identical to the globally optimal
solution.

\begin{figure*}[ht!]
  \centering
  \subfigure[Convergence]{\includegraphics[width=0.49\textwidth]{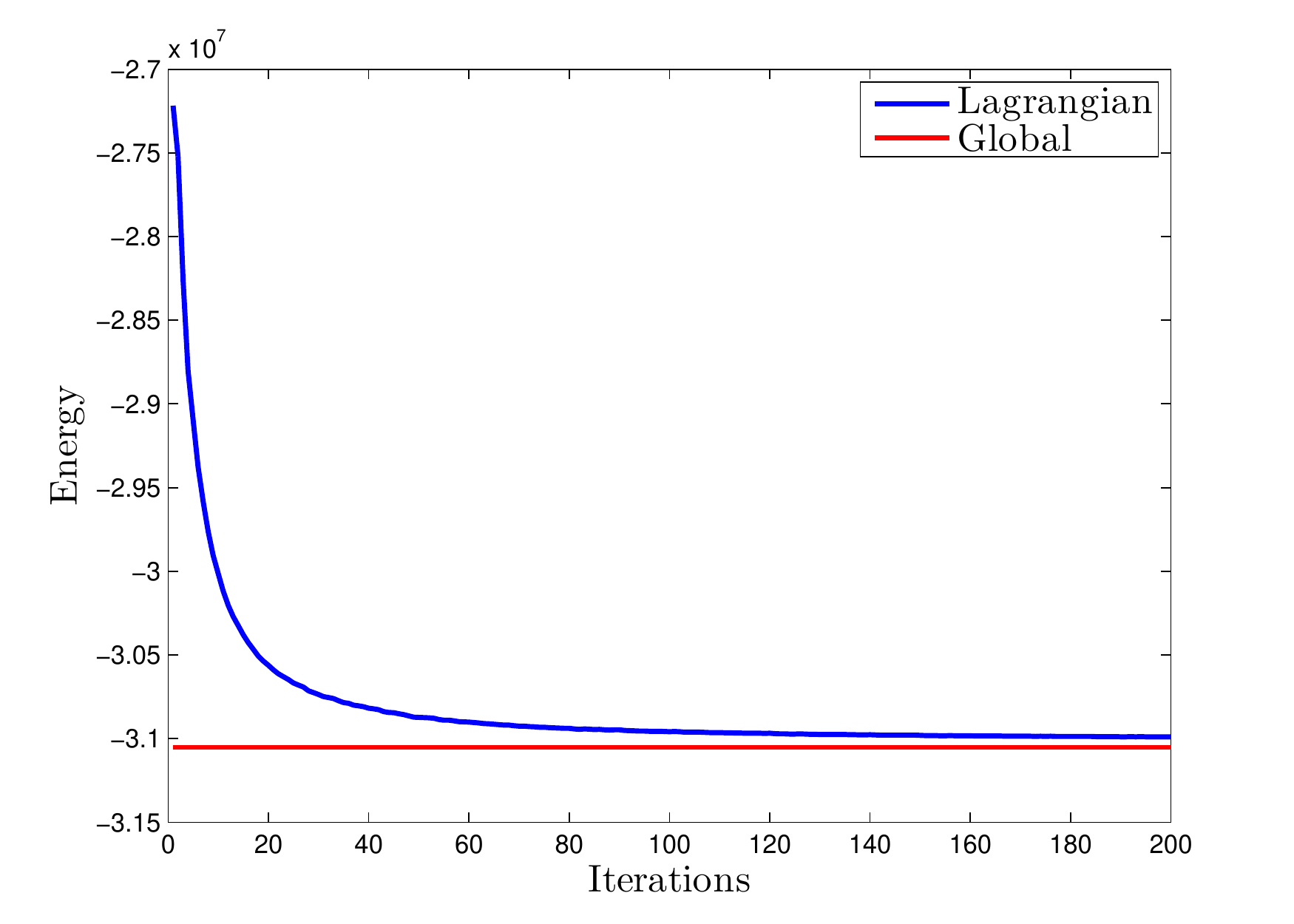}}\hfill
  \subfigure[Globally optimal~\cite{CP2015}]{\includegraphics[width=0.49\textwidth]{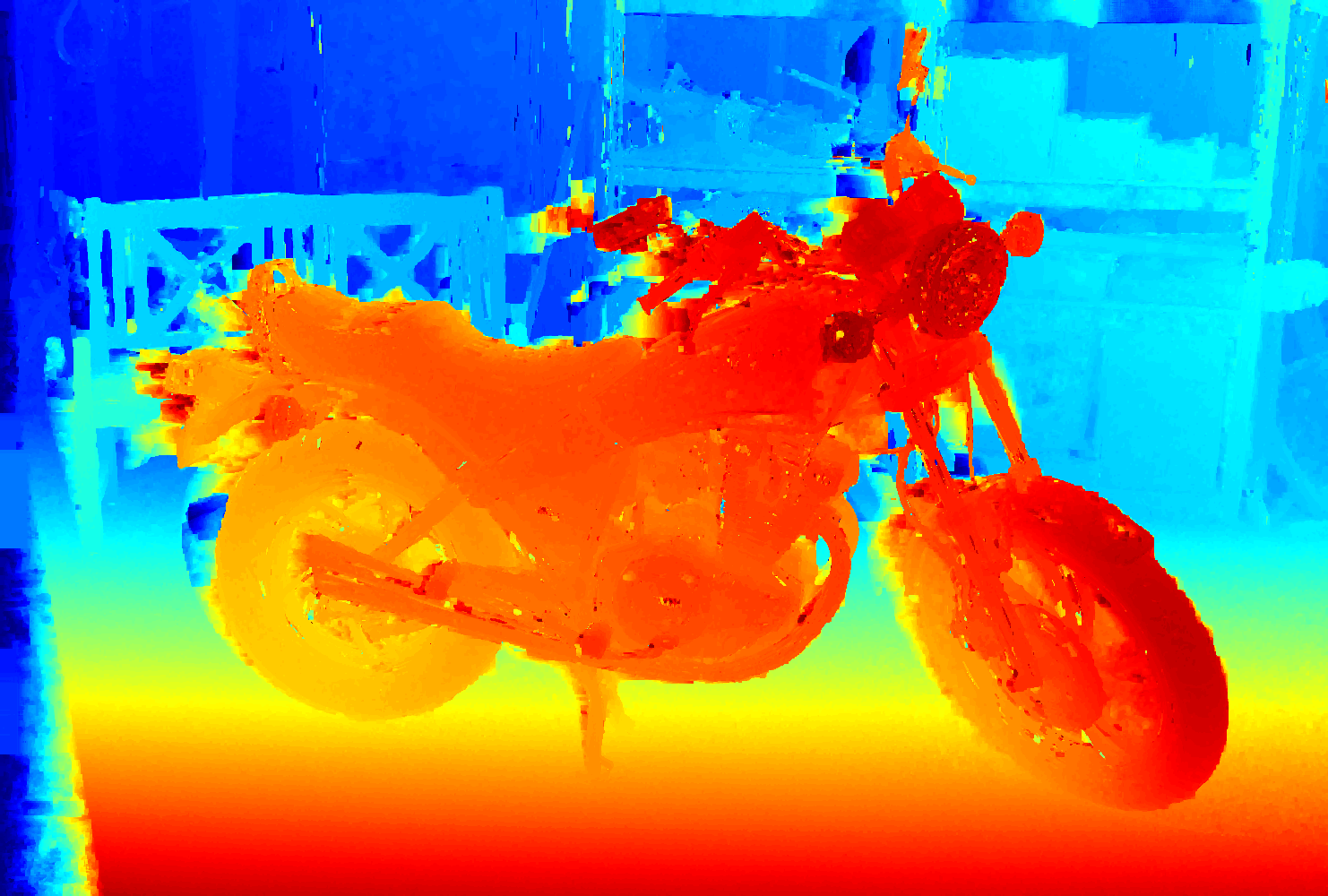}}\\  
  \subfigure[Lagrangian, $k=1$]{\includegraphics[width=0.49\textwidth]{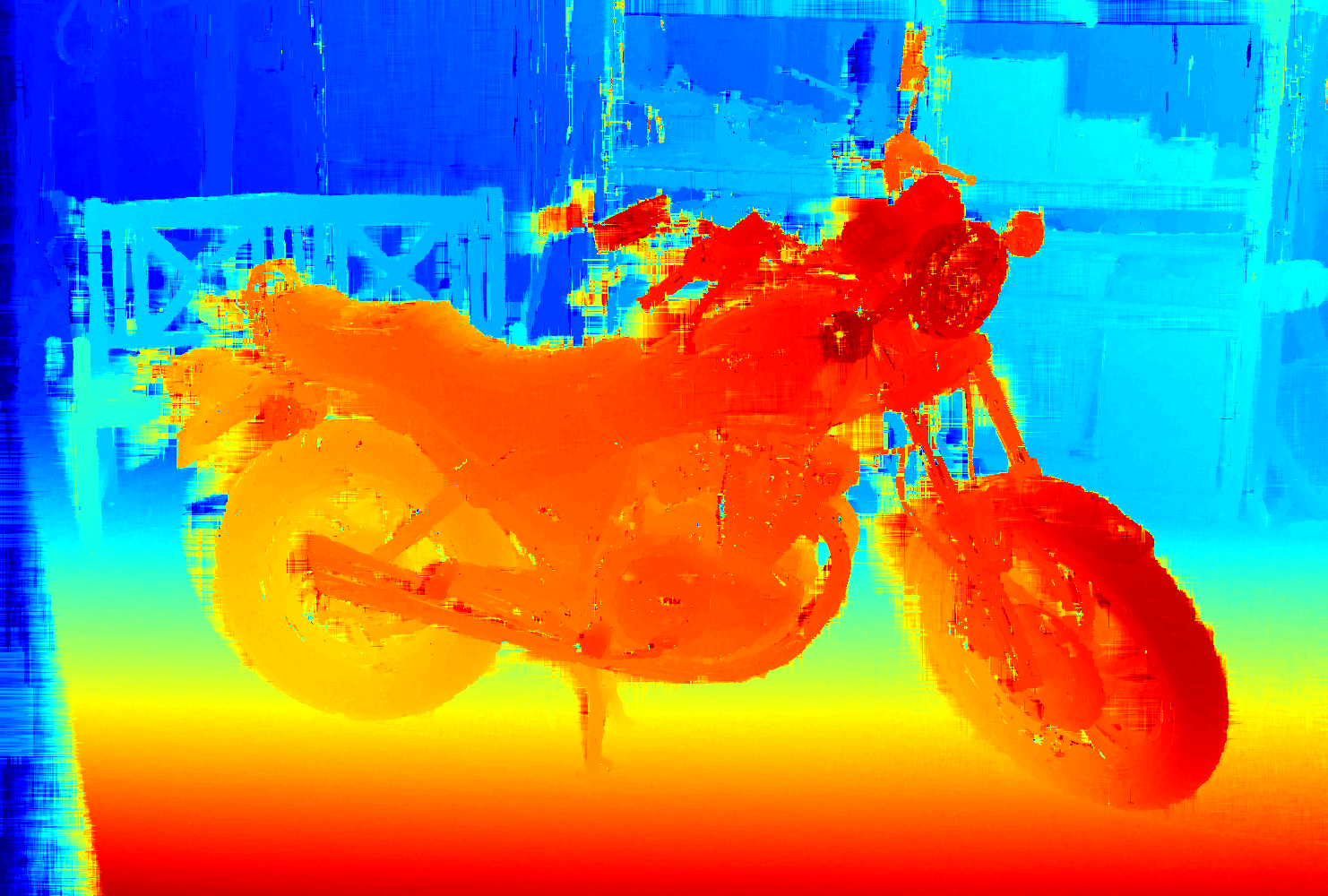}}\hfill
  \subfigure[Lagrangian, $k=10$]{\includegraphics[width=0.49\textwidth]{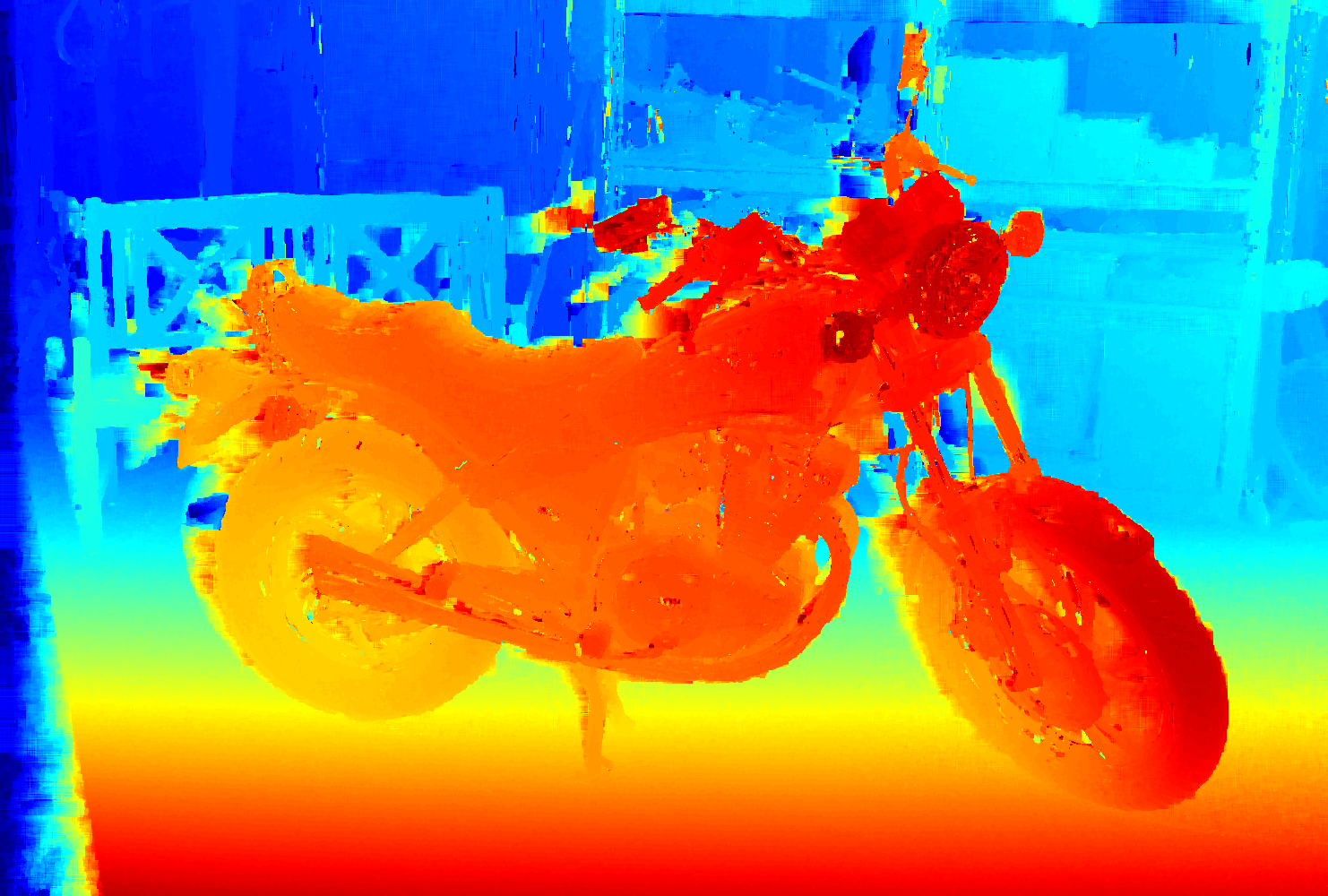}}\\
  \subfigure[Lagrangian, $k=50$]{\includegraphics[width=0.49\textwidth]{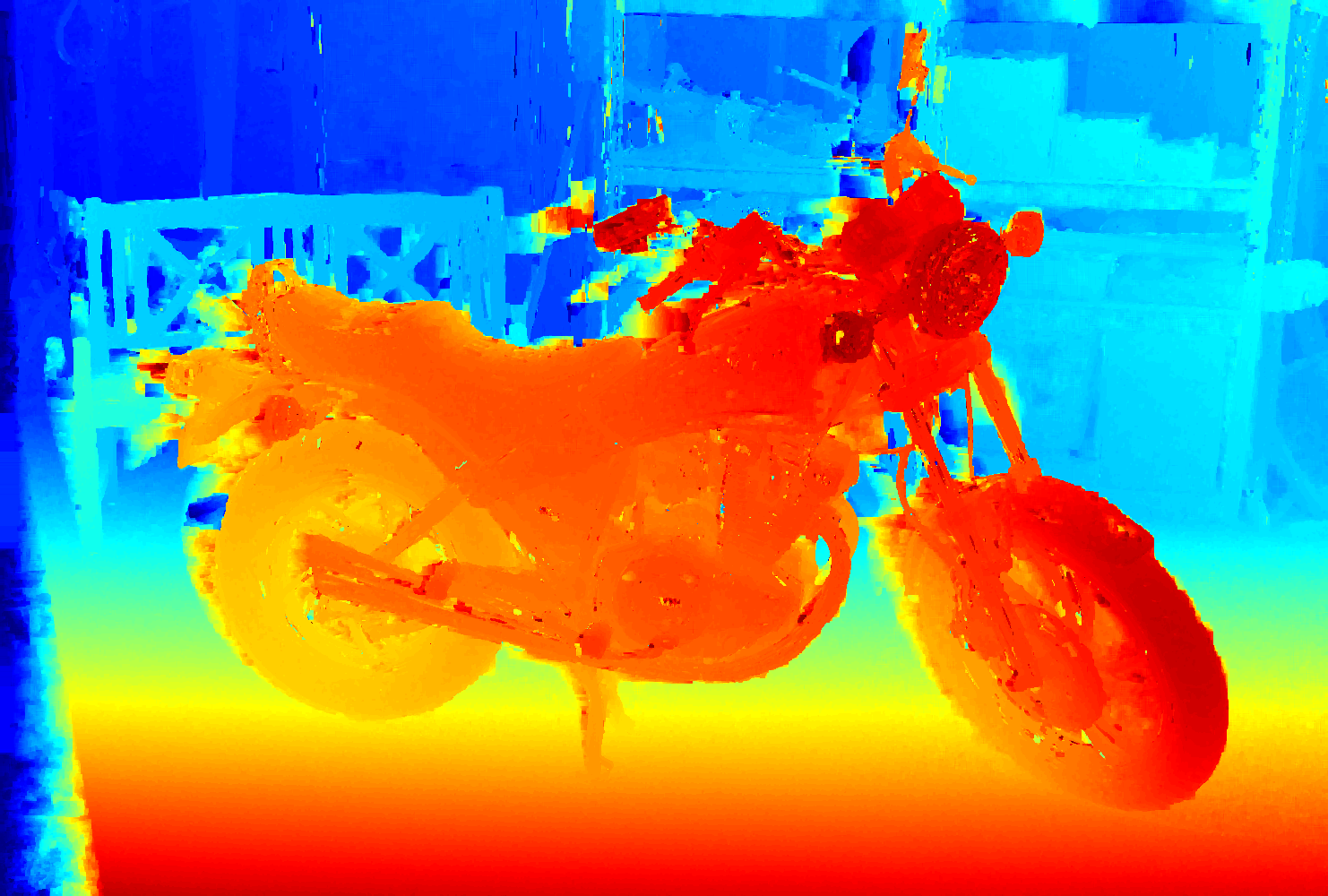}}\hfill
  \subfigure[Lagrangian, $k=100$]{\includegraphics[width=0.49\textwidth]{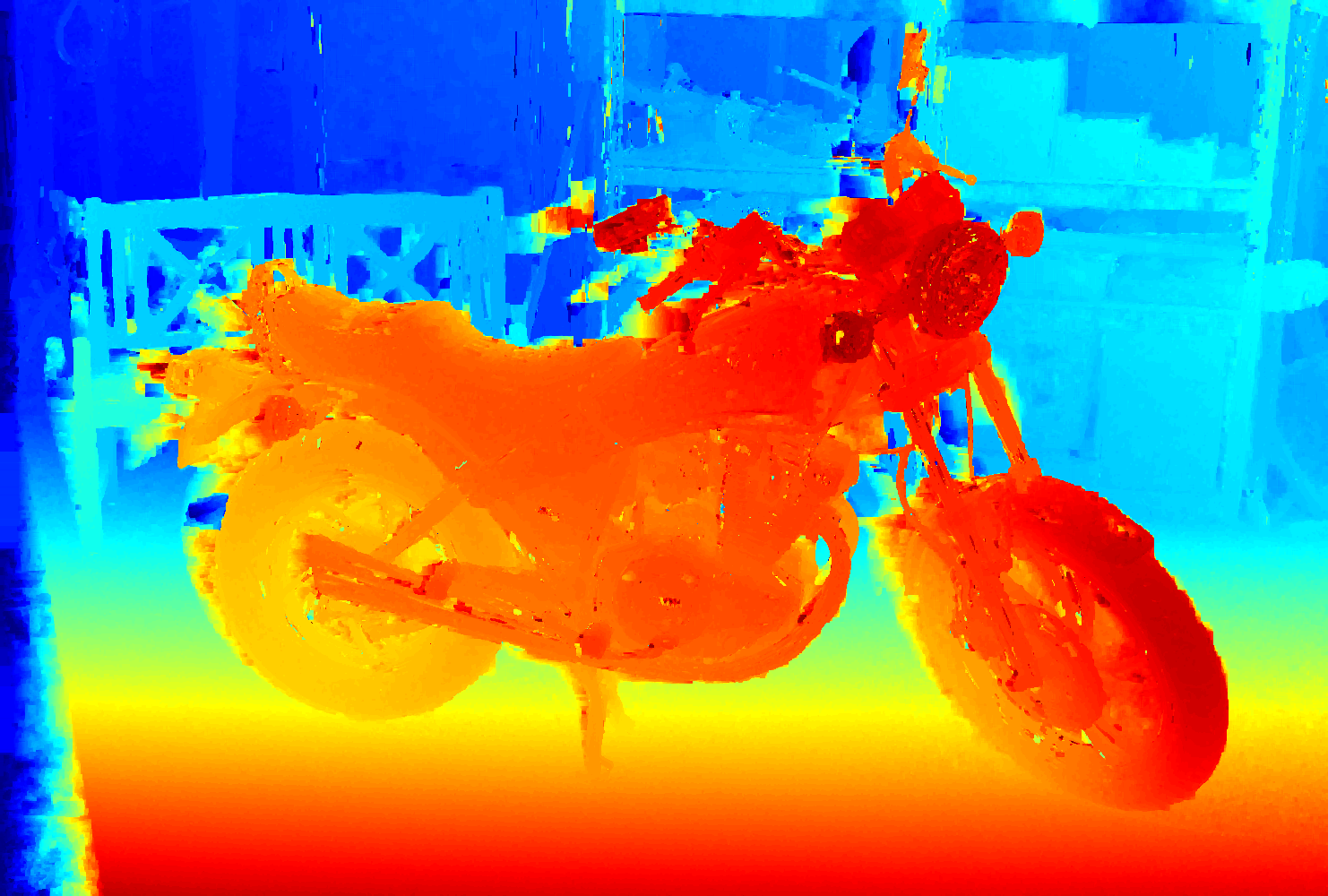}}\\
  \caption{2D Lagrangian decomposition vs. globally optimal solution
    obtained from a 3D lifting~\cite{CP2015}. (a) shows the decrease
    of the primal energy during the iterations of the primal-dual
    algorithm compared to the lower bound obtained from the globally
    optimal solution. (b) is the color coded disparity image from the
    global solution. (c)-(f) are the disparity images $\bar x^k$
    obtained from the proposed Lagrangian decomposition after $k=1$,
    $k=10$, $k=50$, and $k=100$ iterations.}\label{fig:stereo-disp}
\end{figure*}

Fig.~\ref{fig:TV-vs-TTV} finally shows a comparison between convex TV
and nonconvex truncated TV using a truncation value of $C=10$. We also
decreased the strength of the data term by a factor of two in order to
account for the less strong regularization of the TTV. One can see
that the TTV solution yields sharper discontinuities (for example at
the front wheel) and also preserves smaller details (fork tubes). It
is also a bit more sensitive for outliers in the solution, which
however can be removed by some post-processing procedure.

\begin{figure*}[ht!]
  \centering
  \subfigure[TV]{\includegraphics[width=0.49\textwidth]{code/tv-lin/out/TV_disp0100.png}}\hfill
  \subfigure[TTV, $C=10$]{\includegraphics[width=0.49\textwidth]{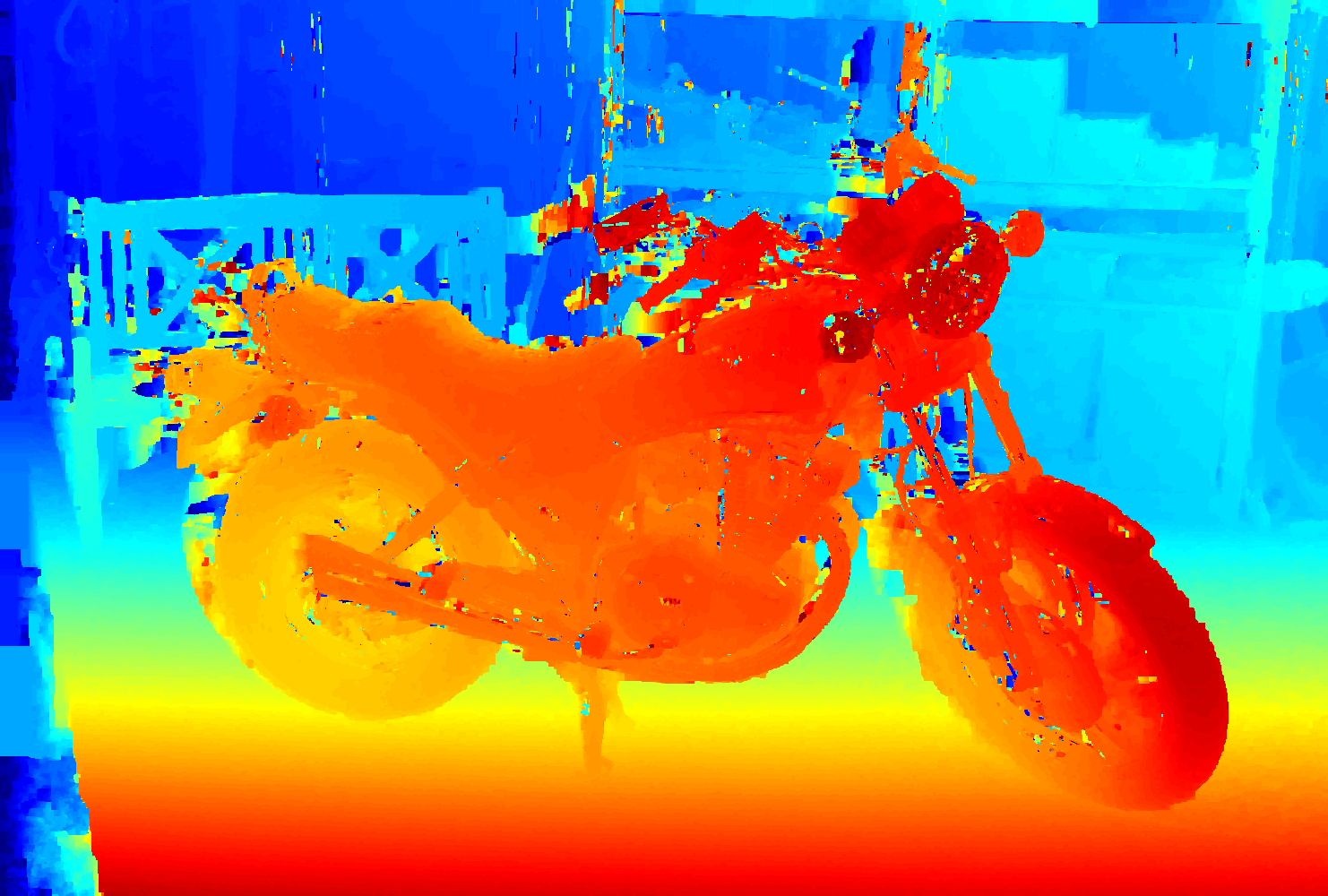}}\\  
  \subfigure[TV-Detail]{\includegraphics[width=0.3\textwidth]{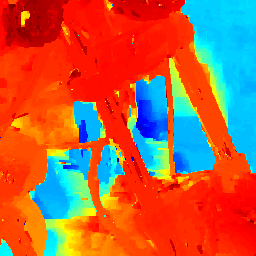}}\hfill
  \subfigure[GT-Detail]{\includegraphics[width=0.3\textwidth]{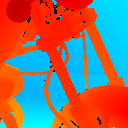}}\hfill
  \subfigure[TTV-Detail]{\includegraphics[width=0.3\textwidth]{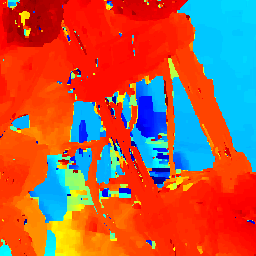}}
  \caption{Qualitative comparions between convex total variation (TV)
    vs. nonconvex truncated total variation (TTV). TTV lieds to
    sharper discontinuities in the solution and better preserves small
    details (see the detail views).}\label{fig:TV-vs-TTV}
\end{figure*}

\section{Conclusion}\label{sec:conclusion}
In this paper we proposed dynamic programming algorithms for
minimizing the total variation subject to different pointwise data
terms on trees. We considered general nonconvex piecewise linear data
terms, convex piecewise linear (and quadratic) data terms and convex
quadratic data terms. 

In case of quadratic data terms, the resulting dynamic programming
algorithm has a linear complexity in the signal length and can be seen
as a generalization of Johnson's method~\cite{Johnson2013} to weighted
total variation. In case of convex piecewise linear data terms, our
dynamic programming algorithm has a worst case complexity of $O(n \log
\log n)$ which improves the currently best performing
algorithm~\cite{DeumbgenKovac:09}. In case of convex piecewise
quadratic unaries we obtain an algorithm with a slightly worse
complexity of $O(n\log n)$ but this algorithm turns out to be useful
for computing proximity operators with respect to 1D TV-$\ell_1$
energies. Finally, in case of nonconvex piecewise linear unaries, we
obtain a worst-case complexity of $O(n^2)$ which turns out to be
useful for approximately solving 2D stereo problems. We evaluated the
dynamic programming algorithms by utilizing them as basic building
blocks in primal-dual block decomposition algorithms for minimizing 2D
total variation models. Our numerical experiments show the efficiency
of the proposed algorithms.

In our block decomposition algorithms all 1D subproblems can be solved
simultaneously, which clearly offers a lot of potential for
parallelization and hence a speedup of the algorithms. We will pursue
this direction in our future work.


\appendix
\section{Proof of Proposition~\ref{prop:ConsecutiveMinConvolution}} \label{sec:ConsecutiveMinConvolution}
It suffices to prove the claim for $m=2$; the general claim will then follow by induction.
Denote $s=h\otimes g$; we need to show that $h^2=s(y)$ for all $y\in \mathbb R$, where $h^1=h\otimes g^1$ and $h^2=h^1\otimes g^2$.

\paragraph{Proof of $h^2(y)\le s(y)$}
First, observe  that $h^2(z)\le h^1(z)\le h(z)$ for any $z$ since $g^1(0)=g^2(0)=0$. For any $x$ we have
\begin{eqnarray*}
h^2(y)&\le& h^1(y) \le h(x) + g^1(y-x) \\
h^2(y)&\le& h^1(x) + g^2(y-x) = h(x) + g^2(y-x)
\end{eqnarray*}
and so
\begin{eqnarray*}
h^2(y)&\le& h(x) + \min_{k\in\{1,2\}} g^k(y-x) = h(x) + g(y-x)
\end{eqnarray*}
Therefore, $h^2(y)\le \min_x [h(x) + g(y-x)] =s(y)$.
\paragraph{Proof of $h^2(y)\ge s(y)$}
For any $x,x'$ we have
\begin{eqnarray*}
s(y)&\!\!\le\!\!& h(x)+g(y-x) \;\le\; h(x)+[g(x'-x)+g(y-x')] \\
    &\!\!\le\!\!& h(x)+g^1(x'-x)+g^2(y-x')
\end{eqnarray*}
where the second inequality holds by the triangle inequality~\eqref{eq:GHAOUSFGHAOUG}.
Therefore, 
\begin{eqnarray*}
s(y)&\!\!\le\!\!& \min_{x,x'} [[h(x)+g^1(x'-x)]+g^2(y-x')] \\
&\!\!=\!\!&\min_{x'} [h^1(x')+g^2(y-x')]\;=\;h^2(y)
\end{eqnarray*}
\paragraph{Mapping $\pi$}
It remains to prove that mapping $\pi=\pi^1\circ\pi^2$ corresponds to min-convolution $h^2=h\otimes g$.
Let $x'=\pi^2(y)$ and $x=\pi^1(x')$, then $h^2(y)=h^1(x')+g^2(y-x')$
\begin{eqnarray*}
h^2(y)&\!\!=\!\!&h^1(x')+g^2(y-x') \\
&\!\!=\!\!&[h(x)+g^1(x'-x)]+g^2(y-x') \\
&\!\!\ge\!\!& h(x)+g(x'-x)+g(y-x') \ge h(x)+g(y-x)
\end{eqnarray*}
(the last inequality is again by~\eqref{eq:GHAOUSFGHAOUG}). The inequality $h(x)+g(y-x)\le h^2(y)$ means that $x\in\argmin_x [h(x)+g(y-x)]$.

\section{Proof of Proposition~\ref{prop:SortingLowerBound}}\label{sec:SortingLowerBound}
Suppose that we are given positive numbers  $b_1,\ldots,b_N$. 
We will construct a chain instance with $n=2N$ nodes
such that values $\lambda^-_{ij}$ will give the input numbers in a sorted
order. This will imply the claim since sorting requires at least $\Omega(n\log n)$ comparisons in the worst case~\cite{Cormen}.

The unary terms for nodes $i=1,\ldots,N$
nodes are given by 
$g_i(z)=\begin{cases}0 & \mbox{if } z\le b_i \\ 1 & \mbox{if } z> b_i \end{cases}$.
The unary terms for nodes $i=N+1,\ldots,n$ are zeros.
The weights for edges $(i,j)\in E$ are set as follows:
$w^-_{ij}=0$ for all edges,
$w^+_{ij}=N+1$ if $i<N$, and
$w^+_{ij}=2N-i-\frac{1}{2}$ if $i\ge N$.

Let us sort numbers $b_1,\ldots,b_N$ in the non-decreasing order, and denote
the resulting sequence as 
 $(c_1,\ldots,c_N)$.
It can be checked that $\lambda^+_{ij}=c_{2N-i}$ for $i=N,\ldots,2N-1$.

\section{Proof of Proposition~\ref{prop:nlogn}}\label{sec:proof:nlogn}

First note that each problem is reduced to solving a finite amount of subproblems of the same type (recall that both $V_0$ and $V_1$ are unions of chains). Let $P_1^0$ stand for the original problem and inductively for $i \geq 1$, let $P^i_1$, \dots, $P^i_{k_i}$ be all the (direct) subproblems of the problems $P^{i-1}_1$, \dots, $P^{i-1}_{k_{i-1}}$. Also, let $\lambda(P)$, $n(P)$ be the number of breakpoints and the number of nodes in subproblem $P$, respectively.
It follows by induction on $i$ that:
\begin{itemize}
\item $\sum_k \lambda(P^i_k) \leq |\Lambda|$ for each $i \geq 0$.
\item $\sum_k n(P^i_k) \leq n$ for each $i \geq 0$.
\item $\lambda(P^i_k) \leq |\Lambda|/2^i$ for each $i \geq 0$ and $k \geq 1$ and hence $i$ is  $O(\log n)$.
\end{itemize}

Since the complexity of dividing problem $P$ into subproblems is $O(\lambda(P)+n(P))$ as described in Section \ref{sec:convex} we now compute the total complexity as

\begin{eqnarray*}
\sum_i \sum_k O(\lambda(P^i_k)+n(P^i_k)) &\!\!\!=\!\!\!& \sum_i O\left(\sum_k \lambda(P^i_k)+n(P^i_k) \right) \\
&\!\!\!=\!\!\!& \sum_i O(n) \subset O(n \log n).
\end{eqnarray*}

\section{Proof of Proposition~\ref{prop:nloglogn}}\label{sec:proof:nloglogn}

We use the same strategy to compute the complexity of the recursive algorithm as in Section \ref{sec:proof:nlogn}. Let $P^0_1$ be the original problem and inductively for $i \geq 1$, let $P^i_1$, \dots, $P^i_{k_i}$ be all the (direct) subproblems of the problems $P^{i-1}_1$, \dots, $P^{i-1}_{k_{i-1}}$. Let $\lambda(P)$, $\calU(P)$, and $\calE(P)$ be the sizes of the sets $\Lambda$, $\calU$, and $\calE$ of problem $P$, respectively.

We claim that:
\begin{itemize}
\item $\sum_k \calU(P^i_k)$ is  $O(|\calU|)$ for each $i \geq 0$.
\item $\sum_k \calE(P^i_k)$ is $O(|\calE|)$ for each $i \geq 0$.
\item $\lambda(P^i_k) \leq |\Lambda|\cdot (3/4)^i$ for each $i \geq 0$ and $k \geq 1$ and hence $i$ is  $O(\log n)$.
\end{itemize}

Since the total number of added vertices (or cut edges) is at most $|\calU|$ (or $|\calE|$), the first two statements follow immediately from induction on $i$. For the third statement recall that the value $\lambda$ was chosen so that no more than $3/4$ breakpoints can end up in any of the sets $\Lambda^+ = \{ \lambda' \in \Lambda: \lambda' > \lambda\}$, $\Lambda^- = \{ \lambda' \in \Lambda: \lambda' < \lambda\}$. Since for each subproblem the set of breakpoints is a subset of one of $\Lambda^+$, $\Lambda^-$ the third statement follows again from induction on $i$.

Since the complexity of dividing problem $P$ into subproblems is $O((\calU(P)+\calE(P))\log m)$, we compute the entire complexity as
\begin{eqnarray*}
&&\sum_i \sum_k O((\calU(P^i_k)+\calE(P^i_k))\log m) \\
&&\quad = \sum_i O\left(\left(\sum_k \calU(P^i_k)+\calE(P^i_k) \right) \log m\right) \\
&&\quad = \sum_i O((|\calU| + |\calE|)\log m)  \subset O(n \log n \log n).
\end{eqnarray*}


{
\bibliographystyle{plain}
\bibliography{lit}
}

\end{document}